\documentclass{article}

\usepackage{microtype}
\usepackage{graphicx}
\usepackage{subfigure}
\usepackage{booktabs} 

\usepackage[accepted]{icml2023} 
\usepackage{hyperref}
\usepackage{packagesList}
\usepackage{times}
\usepackage{natbib}

\usepackage{amsmath}
\usepackage{amssymb}
\usepackage{mathtools}
\usepackage{amsthm}

\theoremstyle{plain}
\newtheorem{theorem}{Theorem}[section]
\newtheorem{remark}{Remark}[section]
\newtheorem*{remark*}{Remark}

\newtheorem{assumption}{Assumption}
\newtheorem{lemma}{Lemma}[section]
\newtheorem{proposition}{Proposition}[section]

\newtheorem*{discussion*}{Discussion}

\theoremstyle{definition}
\newtheorem{definition}{Definition}[section]

\newcommand{\cond}{\textbf{C}\!\!}

\newtheorem*{property*}{Property}

\numberwithin{equation}{section}

\newcommand{\rset}{\mathbb{R}}

\newcommand{\abs}[1]{\left\lvert{#1}\right\rvert}
\newcommand{\norm}[1]{\left\lVert{#1}\right\rVert}

\newcommand{\ffrac}[2]{\ensuremath{\frac{\displaystyle #1}{\displaystyle #2}}}

\newcommand{\PP}[1][]{\ifthenelse{\equal{#1}{}}{\ensuremath{\mathbb{P}}}{\ensuremath{\mathbb{P}\left( #1 \right) }}}
\newcommand{\EE}[1][]{\ifthenelse{\equal{#1}{}}{\ensuremath{\mathbb E}}{\ensuremath{{\mathbb E}\left[ #1 \right]}}}
\newcommand{\Var}[1][]{\ifthenelse{\equal{#1}{}}{\ensuremath{\matrhm{Var}}}{\ensuremath{{\mathrm{Var}}\left[ #1 \right]}}}
\newcommand{\Cov}[1][]{\ifthenelse{\equal{#1}{}}{\ensuremath{\mathrm{Cov}}}{\ensuremath{{\mathrm{Cov}}\left[ #1 \right]}}}

\newcommand{\rloo}{\widehat{\mathcal{R}}_{\mathrm{loo}}}

\DeclareMathOperator*{\argmin}{arg\,min}

\usepackage{xspace}

\newcommand\ie{\emph{i.e.}\xspace}
\newcommand\iid{\ensuremath{\mathit{i.i.d.}}\xspace}

\def\indep{\perp\!\!\!\perp}

\usepackage{xcolor}

\newcommand{\DA}{\textsc{DA}\xspace}
\newcommand{\RERM}{\textsc{RERM}\xspace}
\newcommand{\HTL}{\textsc{HTL}\xspace}
\newcommand{\MSE}{\textsc{MSE}\xspace}

\newcommand{\VC}{\textsc{VC}\xspace}
\newcommand{\loo}{\emph{l.o.o.}\xspace}

\newcommand\eg{\emph{e.g. }\xspace}

\newcommand{\Hy}{\mathcal{H}}
\newcommand{\TR}{\mathcal{R}}
\newcommand{\ER}{\widehat{ \mathcal{R}}}

\newcommand{\risk}{\mathcal{R}}

\newcommand{\data}{\mathcal{D}}
\newcommand{\DD}{\data}

\newcommand{\Kfold}{\textrm{K-fold}}

\newcommand{\alg}{\mathcal{A}}

\newcommand{\loss}{\ell}
\DeclareMathOperator{\sign}{sign}

\usepackage[textsize=tiny]{todonotes}

\icmltitlerunning{Hypothesis Transfer Learning with Surrogate Classification  Losses}

\begin{document}

\twocolumn[
\icmltitle{Hypothesis Transfer Learning with Surrogate Classification  Losses: \\ Generalization Bounds through Algorithmic Stability}




\begin{icmlauthorlist}
\icmlauthor{Anass Aghbalou}{yyy}
\icmlauthor{Guillaume Staerman}{comp}
\end{icmlauthorlist}

\icmlaffiliation{yyy}{Télécom
Paris, Institut Polytechnique de Paris, LTCI, Palaiseau,
France.}
\icmlaffiliation{comp}{Université Paris-Saclay, Inria, CEA, ´
Palaiseau, 91120, France}

\icmlcorrespondingauthor{Anass Aghbalou}{anass.aghablou@telecom-paris.fr}

\icmlkeywords{Algorithmic, ICML}

\vskip 0.3in
]

\printAffiliationsAndNotice{}  
	
\bibliographystyle{icml2023}

\begin{abstract}
\textit{Hypothesis transfer learning} (\HTL) contrasts domain adaptation by allowing for a previous task leverage, named the source, into a new one, the target, without requiring access to the source data. Indeed, \HTL relies only on a hypothesis learnt from such source data, relieving the hurdle of expansive data storage and providing great practical benefits. Hence, \HTL is highly beneficial for real-world applications relying on big data. The analysis of such a method from a theoretical perspective faces multiple challenges, particularly in classification tasks. This paper deals with this problem by studying the learning theory of \HTL  through algorithmic stability, an attractive theoretical framework for machine learning algorithms analysis. In particular, we are interested in the statistical behaviour of the regularized empirical risk minimizers in the case of binary classification. Our stability analysis provides learning guarantees under mild assumptions. Consequently, we derive several complexity-free generalization bounds for essential statistical quantities like the training error, the excess risk and cross-validation estimates. These refined bounds allow understanding the benefits of transfer learning and comparing the behaviour of standard losses in different scenarios, leading to valuable insights for practitioners.


\end{abstract}

\section{Introduction}
Traditional supervised machine learning methods share the common assumption that training data and test data are drawn from the same underlying distribution. However, this assumption is often too restrictive to hold in practice. In many real-world applications, a hypothesis is learnt and deployed in different environments that exhibit a distributional shift. A more realistic assumption is that the marginal distributions of training (\emph{source}) and testing (\emph{target}) domains are different but related. This is the framework of \emph{domain adaptation} (\DA), where the learner is provided little or no labeled data from the target domain but a large amount of data from the source domain. This problem arises in various real-world applications like natural language processing \cite{dredze2007frustratingly,ruder2019transfer}, sentiment analysis \cite{blitzer2007biographies,liu2019survey}, robotics \cite{zhang2012generalization,bousmalis2018using} and many other areas.

Several works shed light on the theory of \DA \cite{blitzer2007learning,mansour2009domain,ben2010theory,zhang2012generalization,cortes2015adaptation,zhang2019bridging} and suggest schemes that generally rely on minimizing some similarity distances between the source and the target domains. However, the theoretical analysis shows that a \DA procedure needs many unlabeled data from both domains to be efficient. Besides, even when unlabeled data are abundant, minimizing a similarity distance can be time-consuming in many scenarios.

To tackle this practical limitation, a new framework that relies only on the source hypothesis was introduced, the so-called \emph{hypothesis transfer learning} (\HTL) \cite{li2007bayesian,orabona2009,kuzborskij2013,perrot2015theoretical,kuzborskij2017fast,du2017hypothesis}.
\HTL is tailored to the scenarios where the user has no direct access to the source domain nor to the relatedness between the source and target environments. As a direct consequence, HTL does not introduce any assumptions about the similarity between
the source and target distributions. It has the advantage of not storing abundant source data in practice.

In this work, we analyze \HTL through Regularized Empirical Risk Minimization (\RERM) in the binary classification framework. Our working assumptions encompass many widely used \emph{surrogate} losses, such as the exponential loss used by several boosting algorithms like AdaBoost \cite{freund1997decision}, the logistic loss, the softplus loss, which serves as a smooth approximation of the hinge loss \cite{dugas2000incorporating}, the mean squared error (\MSE) and the squared hinge that represents the default losses for least squares/modified least squares algorithms \cite{rifkin2003regularized}. The attractive quality of these surrogate losses is that they are \emph{classification calibrated} \cite{Zhang2004,bartlett2006convexity}. In other words, they represent a convex upper bound for the classification error and minimizing the expected risk regarding a surrogate loss yields a predictor with sound accuracy.

This paper's theoretical analysis uses the notion of \emph{algorithmic stability}. Formally, assuming that one has access to a small labeled set, we derive many complexity-free generalisation bounds that depend only on the source hypothesis's quality. In particular,  such an analysis allows us to compare the behavior of different losses in different scenarios and to answer some practical questions such as: \emph{which surrogate loss is recommended when the source and target domains are related? Which surrogate loss is robust to heavy distribution shift?} 

The notion of algorithmic stability and its consequences in learning theory has received much attention since its introduction in \cite{DEvroy-79}. It allows obtaining complexity-free generalization bounds for a large class of learning algorithms such as k-nearest-neighbours \cite{DEvroy-79}, empirical risk minimizers \cite{kearns1999algorithmic},  Support Vector Machine  \cite{bousquet2002stability}, Bagging \cite{elisseeff05a}, \RERM \cite{Zhang2004,wibisono2009sufficient}, stochastic gradient descent \cite{hardt2016train},  neural networks with a simple architecture \cite{charles18a}, to name but a few. For an exhaustive review of the different notions of \emph{stability} and their consequences on the generalization risk of a learning algorithm, the reader is referred to \cite{kutin2002}.

  Only a few works derive theoretical guarantees for \RERM in the \HTL  framework and are all formalized in a regression setting. A stability analysis has been provided for the HTL algorithm in the case of RLS for regression in \citet{kuzborskij2013} limited to the least-squares loss. Later,  \citet{kuzborskij2017fast} considered the class of smooth losses and obtained statistical rates on the empirical risk, being a particular case of the stability guarantees. However, this smoothness assumption may be considered strong since it is not satisfied for hypotheses learnt from the exponential loss or vacuously satisfied for hypotheses learnt from the softplus loss. Besides, \citet{du2017hypothesis} proposed a novel algorithm to adapt the source hypothesis to the target domain. Nonetheless, the theoretical guarantees they derived are obtained with several strong assumptions, unverifiable in practice. The obtained bounds depend on many unknown parameters (for further details, see Section~\ref{sec:stability}, where all these assumptions are explicitly listed and discussed). Other theoretical results studying  \HTL outside the framework of \RERM  can be found \cite{li2007bayesian,morvant2012parsimonious,perrot2015theoretical,dhouib2018revisiting}. However,  most of these theoretical results depend on a complexity/distance measure or/and are valid on a different framework than classification. For example, \citet{perrot2015theoretical} explores the notion of algorithmic stability in  \emph{metric learning} with Lipschitz
loss functions to study the excess risk of some  algorithms. The obtained bounds are not intuitive as they depend on the  Lipschitz constant and cannot be easily extended to many usual classification losses. Furthermore, the proof techniques in the latter work are far from ours. 

 
On the other hand, when the source is known, many theoretical guarantees can be found in the domain adaptation literature, see e.g. \citet{mansour2009domain,ben2010theory,zhang2012generalization,cortes2015adaptation} and \citet{zhang2019bridging}, among others. Their rates involve the complexity of the hypothesis class and the distance between the source and the target distribution that may be unknown in practice and drastically deteriorate the rates. 

Another related subject is \emph{meta learning}, broadly described as leveraging data from pre-existing tasks to derive algorithms or representations that yield superior results on unencountered tasks. Many theoretical works such as  \cite{Khodak2019,balcan2019,denevi19a} or \cite{Denevi2020} have studied this problem. Yet, the obtained theoretical guarantees in the latter works depend on the smoothness parameters of the loss function and the regularizers. The proof techniques from the present paper can be incorporated into the proof of the latter references to obtain more sharp and intuitive learning bounds, that is, bounds exclusively depending on the quality of the source hypothesis.
\paragraph{Contributions} In this paper, we investigate the statistical risk of the hypothesis transfer learning procedure dedicated to the binary classification task. To that end, we adopt the angle of algorithmic stability that offers an appealing theoretical framework to analyze such a method. This is the first work exploring algorithmic stability for \HTL  with the usual classification loss 
functions. In this paper, we provide a (pointwise) hypothesis stability analysis of the \HTL in the classification framework for any losses satisfying mild conditions. Furthermore, we show that our main assumptions are valid for the most popular classification losses and derive their associated constants. Based on these stability results, we investigate the statistical behavior of the generalization gap and the excess risk of the \HTL procedure. We provide an intuitive finite-sample analysis of these quantities and highlight the statistical behavior of common losses.



\section{Background and Preliminaries}\label{sec:framework}
In this section, we start by recalling the framework of Hypothesis transfer learning  and describe the concept of stability.

\subsection{Hypothesis Transfer Learning}

Considering the source and target domains, hypothesis transfer learning leverages the learnt hypothesis with the source dataset,  without having access to the raw source data or any information between source and target domains, to solve a machine learning task on the target domain. Formally, we denote by $\mathcal{Z}_S$ and  $\mathcal{Z}_T$ the source and target domains and assume that we have access  to $n \in \mathbb{N}, n\geq 1$ i.i.d. observations  $\mathcal{D}_T =Z_1,\ldots, Z_n \in \mathcal{Z}_T$ with a distribution $P_T$ lying in the target domain and a source hypothesis $h_S$ learnt from $m \in \mathbb{N}, m\geq 1$ i.i.d. observations $\mathcal{D}_S = Z_1^S,\ldots, Z_m^S \in \mathcal{Z}_S$ drawn from the source distribution $P_S$.  In the HTL framework, we do not have access to the source observations but only to the resulting source hypothesis $h_S$. It is worth noting that $n\ll m$ in many practical scenarios. In this paper,  we focus on the binary classification task. Therefore, our domains consist of a Cartesian product of a source/target covariate space $\mathcal{X}_S/\mathcal{X}_T$ and the set $\{-1, 1\}$, i.e. $\mathcal{Z}_S = \mathcal{X}_S \times \{-1, 1\}$ and $\mathcal{Z}_T = \mathcal{X}_T \times \{-1, 1\}$. In  addition, we assume that $\mathcal{X}_T\subset \mathcal{X}_S\subset \rset^d$. Consider two classes of hypotheses $\Hy_S$ and $\Hy_T$, an \HTL algorithm aims to use a source hypothesis $h_S \in \Hy_S $  learnt on $\DD_S$ to improve the performance of a classification algorithm over $\DD_T$. Precisely, it is defined as a map
\begin{align*}
\alg:\left(\mathcal{Z}_T\right)^n\times \mathcal{H}_S
&\rightarrow \mathcal{H}_T\\ 
\left(\DD_T, h_S\right) & \mapsto h_T.
\end{align*}

Throughout the paper, we assume that $h_S$ is given and fixed, and we use the shorthand notation  $\mathcal{A}(\DD_T)$ instead of $\mathcal{A}(\DD_T,h_S)$ for the sake of clarity. 

Let $\loss: \Hy_T \times \mathcal Z_T \mapsto\rset_+$ denote a loss function so that $\loss(h_T,Z)$ is the error of $h_T \in \Hy_T$ on the observation $Z=(X,Y)\in \mathcal Z_T$. In this work, we assume that $\loss(h_T,Z)=\phi\left(h_T(X)Y\right)$ for some non negative convex function $\phi$. The generalization risk of the predictor $\alg(D_T)$ is denoted by
\begin{align*}
\TR\big[\alg\left(\DD_T\right)\big]&=\EE_{Z\sim P_T}\left[ \loss\left(\alg\left(\DD_T\right),Z\right) \right]\\
&=\EE\left[ \loss\left(\alg\left(\DD_T\right),Z\right)\mid \DD_T \right].
\end{align*}

 Notice that the randomness in the latter expectation stems from the novel observation $Z$ only while the trained algorithm $\alg(\DD_T)$ is fixed. Its empirical counterpart, the \emph{training} error of $\alg\left(\DD_T\right)$  writes as 
\[\ER\big[\alg(\DD_T)\big]=\frac{1}{ n} \sum_{i=1}^n \loss(\alg(\DD_T),Z_i).
\]
The latter estimate is known to be optimistic since most learning algorithms are conceived to minimize the training loss. Thus, a more reliable estimate would be the \emph{deleted} estimate or the so-called leave-one-out (\loo) estimate:

\begin{equation}\label{def:risk-loo}
\rloo\big[\alg(\DD_T)\big]=
\frac{1}{ n} \sum_{i= 1}^{n}\loss\left(\alg(\DD_T^{\backslash i}) , Z_i\right),
\end{equation} 
where $\DD_T^{\backslash i}=\DD_T\setminus \left\{Z_i\right\}$ denotes the dataset $\DD_T$ with the $i$'th element removed.\\

\begin{remark}[\textsc{accelerated \loo}]
At first sight, one can notice that computing the \loo risk measure is a heavy task in practice since one needs to train the algorithm $n$ times. However, in our case, one can use the closed form formula of the \loo estimate for \textsc{RERM} algorithms derived in \citet{wang2018approximate}.
\end{remark}

\subsection{Algorithmic Stability}


In this part, we briefly recall important notions of stability that will be used in the paper. The notion of \emph{stability} was first introduced in \citet{DEvroy-79} to derive non-asymptotic guarantees for the leave-one-out estimate. Let denote by $[n]$ the set of indices $\{1, \ldots, n\}$. The algorithm $\alg$ is called stable if removing a training point $Z_i$, $i\in [n]$,  from the $\DD_T$  or replacing $Z_i$ with an independent observation $Z'$ drawn from the same distribution does not alter the risk of the output. Later, \citet{bousquet2002stability} introduced the strongest notion of stability, namely \emph{uniform stability}, an assumption used  to derive probability upper bounds for the training error and the \loo estimate \cite{bousquet2002stability,elisseeff05a,hardt2016train,bousquet2020sharper,nikita2021stability}. Equipped with the above notations, \emph{uniform} stability, also called \emph{leave-one-out} stability, can be defined as follows.

\begin{definition}\label{def:unif-stable}
	The algorithm $\alg$ is said to be   $\beta(n)$-\emph{uniformly} stable with respect to a loss function $\loss$ if, for any $i\in [n]$ and $Z\in \mathcal{Z}_T$, it holds:  
	\begin{equation*}
	\left|\loss\left(\alg(\DD_T),Z\right)-
	\loss\left(\alg(\DD_T^{\backslash i}),Z\right)\right| \leq \beta(n). 
	\end{equation*}
\end{definition}

In practice, uniform stability may be too restrictive since the
bound above must hold for all $Z$, irrespective of its marginal
distribution. While weaker, the following notion of stability
is still enough to control the leave-one-out deviations \cite{DEvroy-79,bousquet2002stability,elisseeff05a,kuzborskij2013}.
\begin{definition}\label{def:hypothesis-stable}
	The algorithm $\alg$  has a \emph{hypothesis} stability $\beta(n)$ with respect to a loss function $\loss$ if, for any $i\in [n]$, it holds:  
	\begin{equation*}
	\left\|\loss\left(\alg(\DD_T),Z\right)-
	\loss\left(\alg(\DD_T^{\backslash i}),Z\right)\right\|_1 \leq \beta(n),
	\end{equation*}
	where $\left\|X \right\|_q=\left(\EE\left[\left|X\right|^q\right]\right)^{1/q}$ is the $L_q$ norm of $X$.
\end{definition}
We now recall a direct analogue of hypothesis stability: the \emph{pointwise hypothesis stability}. The latter property is used to derive \textrm{PAC} learning bounds for the training error \cite{bousquet2002stability,elisseeff05a,charles18a}.
\begin{definition}\label{def:pointwise-hypothesis-stable}
The algorithm $\alg$  has a \emph{pointwise hypothesis } stability  $\gamma(n)$ with respect to a loss function $\loss$ if, for any $i\in [n]$, it holds:  
	\begin{equation*}
	\left\|\loss\left(\alg(\DD_T),Z_i\right)-
	\loss\left(\alg(\DD_T^{\backslash i}),Z_i\right)\right\|_1 \leq \gamma(n).
	\end{equation*}
	
\end{definition}
Note that the approach based on stability does not refer to a complexity measure like the \VC dimension or the Rademacher complexity. There is no need to prove uniform convergence,
and the generalization error (cf. Equation \ref{def:gen-gap} below) depends directly on the stability parameter. Our work aims to use the notion of algorithmic stability to derive sharper bounds for the \HTL problem. More precisely, the magnitude of the obtained bounds is directly related to the quality of $h_S$ on the target domain \big(represented by $\risk[h_S]$\big) instead of the complexity of the hypothesis class \cite{ben2010theory,zhang2012generalization,cortes2015adaptation,zhang2019bridging}.
\subsection{Working Framework}

This paper analyses hypothesis transfer learning through regularised empirical risk minimization (\RERM). In particular,  it includes the popular Regularized Least Squares (RLS) with biased regularization \citep{orabona2009} that has been analyzed in \citet{kuzborskij2013} and \citet{kuzborskij2017fast}. Formally, we consider the following algorithm $\alg$ such that:
\begin{equation}\label{def:algo-transfer}
    \alg (\DD_T, h_S)= \hat{h}(\cdot \; ;  \DD_T)+ h_S(\cdot), 
\end{equation}
where the function $\hat{h}: \mathbb{R}^d \rightarrow \mathbb{R}$ is obtained from the target set of data via the minimization problem:
\begin{align}\label{def:algo-training-transfer}
\hat{h}&= \argmin_{h \in \mathcal{H}}\frac{1}{n}\sum_{i=1}^n\phi\left(\left(h\left(X_i\right)+h_S\left(X_i\right)\right)Y_i\right)+\lambda\lVert h\rVert_k^2\nonumber\\
&=\argmin_{h \in \mathcal{H}}\ER(h+h_S)+\lambda\lVert h\rVert_k^2,
\end{align}

with the family of hypotheses $\mathcal{H}$ being a reproducing kernel Hilbert space (\textsc{RKHS}) endowed with a kernel $k$, an inner product $\langle \cdot,\cdot\rangle$ and a norm $\lVert\cdot\rVert_k$. The resulting map arising from the HTL is the sum of the source hypothesis $h_S$ and the target hypothesis $\hat h$ where $\hat h$ is learnt involving the source map.

It is worth noting that our analysis encompasses the least square with  biased regularization  \citep{scholkopf2001generalized,orabona2009} commonly studied in transfer learning \citep{kuzborskij2013,kuzborskij2017fast}, briefly recalled below.

\begin{remark}[\textsc{link with RLS}]\label{remark:rls-link}

The \textsc{RLS} with biased regularization is a particular case of the proposed algorithm~\ref{def:algo-transfer}. Indeed,  by choosing $k$ as the linear kernel $k(x_1,x_2)=x_1^\top x_2$ and the loss $\phi(x)=\left(1-x\right)^2$, it is equivalent to
	$$\alg= \hat{h}+h_S,$$
	with $\hat{h}(x)= \hat{u}^\top x$ and
 \begin{align}\label{def:rls-regression}
    \hat{u}&=\argmin_{u \in \rset^d}\frac{1}{n}\sum_{i=1}^{n}\left(u^\top X_i +h_S(X_i)-Y_i\right)^2+\lambda\lVert u\rVert_2^2.
 \end{align}
 Furthermore, if $h_S(x)=v^\top x$  is a linear classifier with $v\in \mathbb{R}^d$, then
 \begin{align*}
    \hat{u}&=\argmin_{u \in \rset^d}\frac{1}{n}\sum_{i=1}^{n}\left(u^\top X_i -Y_i\right)^2+\lambda\lVert u-v\rVert_2^2,
 \end{align*}
 which is the original form of  biased regularisation algorithms \cite{scholkopf2001generalized,orabona2009}. See Appendix~\ref{subsec:app:RLS} for technical details.
\end{remark}

\section{Stability Analysis}\label{sec:stability}

The subsequent analysis requires technical assumptions, listed below. We assume that the source hypothesis and the kernel $k$ are bounded, as stated in the following assumptions.

\begin{assumption} \label{Ass:1}The source hypothesis is bounded on the target space:  
\begin{equation*}
\norm{h_S}_\infty=\sup_{x\in \mathcal{X}_T}\abs{h_S(x)}<\infty.
\end{equation*}
\end{assumption}

\begin{assumption}\label{Ass:2} The kernel $k$ is bounded:
\begin{equation*}
    \sup_{x_1, x_2\in \mathcal{X}_T} k(x_1,x_2)\leq \kappa.
\end{equation*}
\end{assumption}
The boundness of the kernel is a common and mild assumption (see e.g. \citealp{bousquet2002stability,Zhang2004,wibisono2009sufficient}). It is satisfied by many usual kernels like the Gaussian kernel and the sigmoid kernel. Furthermore,  when $\mathcal{X}_T$ is bounded, then polynomial kernels are also bounded.

We now investigate the accuracy of the HTL proposed framework and provide general stability results under slight assumptions. Furthermore, we show that these assumptions are satisfied by most of the popular ML surrogate losses used in practice and derive precisely the associated constants involved in our theoretical results.


\subsection{Hypothesis Stability}\label{subsec:stab}
This section analyzes the hypothesis stability of general surrogate ML losses for the proposed HTL framework. To study the stability of Algorithm~\ref{def:algo-transfer}, we start by showing that the solution of the optimization problem~\ref{def:algo-training-transfer} lies in the sphere with a data-driven radius, as stated in the following lemma.

\begin{lemma}\label{lemma:ball}
Suppose that Assumptions~\ref{Ass:1} and~\ref{Ass:2} are satisfied. Then the solution of Equation~\eqref{def:algo-training-transfer} lies in the set 
$\left\{h\in\mathcal{\mathcal{H}}, \;  \lVert h \rVert_\infty \leq  \hat r_\lambda \right\}$ with
$$\hat r_\lambda= \kappa\sqrt{\alpha\ER\left[h_S\right]}, $$ where $\alpha= \kappa / \lambda$ \hspace{0.01cm}. 
\end{lemma}

\begin{proof}
    The proof is postponed in the Appendix~\ref{subsec:app:lemma3.1}. 
\end{proof}

This lemma ensures that the norm of the solution of the optimisation problem \ref{def:algo-training-transfer}  decreases when the quality of $h_S$ increases. In the rest of the paper, for a given index $i\in [n]$, we denote by $\hat r^i_\lambda= \kappa\sqrt{\alpha\ER^{\backslash i}\left[h_S\right]}$, $\ER^{\backslash i}$ the training error with the $i$'th sample removed and $\hat \rho^i_\lambda= \max\left(\hat r_\lambda,\hat r^i_\lambda\right)$.

Before stating our main theorem, we first require an additional assumption involving the empirical radius obtained in Lemma~\ref{lemma:ball}. 

\begin{assumption} \label{Ass:3} \label{property:derivative-expectation}
The function $\phi$ is  differentiable and convex. Furthermore, $\forall i\in[n]$, it holds:
\begin{align*}
	\EE\left[\sup_{\abs{y'},\abs{y}\leq \hat \rho^i_\lambda} \left|\phi'(h_S(X')Y'+y')\phi'(h_S(X)Y+y)\right|\right]&\\
	&\hspace{-25mm}\leq \Psi_1\left(\risk\left[h_S\right]\right),
\end{align*}
where $Z=(X,Y)$, $Z'=(X',Y')$ are two samples drawn from $P_T$ independent of $\DD^{\backslash i}$  and $\Psi_1$ is a decreasing function  verifying $\Psi_1(0)=0$.
\end{assumption}

The bound stated in the theorem below reveals the generalisation properties of the presented HTL procedure through the stability framework. 

\begin{proposition}\label{prop:stability-main}
	Suppose that Assumptions~\ref{Ass:1},~\ref{Ass:2}  and~\ref{Ass:3} are satisfied. Then the algorithm $\alg$ (cf. Equation~\eqref{def:algo-transfer}) is  hypothesis stable with parameter 
	 $$\beta(n)=\ffrac{\alpha\left(\Psi_1\left(\risk\left[h_S\right]\right)  \wedge \norm{\phi'}_\infty^2\right)}{ n}.$$
\end{proposition}
\begin{proof}
The proof is  postponed to the Appendix~\ref{subsec:app:prop3.1}.
\end{proof}

We obtain a stability rate of order $\mathcal{O}\left(\frac{\Psi_1\left(\risk\left[h_S\right]\right)\alpha}{n} \right)$ for any  losses satisfying Assumption~\ref{Ass:3}. It naturally depends on the risk of the source classifier, where the expectation is taken on the target data distribution. Therefore, the source task directly influences the rate of the HTL classifier. The standard stability rate of \RERM without transfer learning (without source) is of order $\mathcal{O}(\alpha/n)$, see Theorem 4.3 in \citet{Zhang2004} or  Theorem 3.5 in \citet{wibisono2009sufficient}. A relevant source hypothesis
allows us to obtain faster rates than in standard \RERM. Thus, one can directly notice the benefits of using a \emph{good} source hypothesis on the stability of \RERM. 
The negative transfer, i.e. the source hypothesis has a negative effect and deteriorates the target learner, is analyzed and discussed in Section~\ref{subsec:gen}.
\\

\begin{remark}[\textsc{Related Work}]
The only existing result studying hypothesis stability in HTL is in \citet{kuzborskij2013}. However, the analysis is only in a regression framework with the mean squared error loss. The proof techniques in
\citet{kuzborskij2013} rely heavily on the closed-form formulas of the ordinary least square estimate, which does
not hold in a general setting like ours. Furthermore, we obtain equivalent (up to constants) stability rates as in \citet{kuzborskij2013}. More details are given in Section~\ref{subsec:exhib} where we explicit constants $\Psi_1$ for most of popular losses. 
\end{remark}

\noindent \textbf{Existing assumptions in \DA and  \HTL literature}
Statistical guarantees obtained in these fields generally assume that the loss function verifies a smoothness condition. For example, in \citet{mansour2009domain} and \citet{cortes2015adaptation}, their analysis supposes that $\loss$ verifies the triangle inequality, which holds only for the \MSE  and squared hinge. Moreover, the obtained upper bounds in these works depend on the complexity of $\mathcal{H}$ and some discrepancy distances between the source and target distributions $P_S$ and $P_T$, which deteriorates the statistical rates. In \citet{kuzborskij2017fast}, they suppose that the derivative of the loss is Lipschitz which is not the case for the exponential. Furthermore, even if the loss satisfies this smoothness assumption, their constants depend heavily on the smoothness parameter, and it would yield vacuous bounds in many practical situations. For example, the softplus function $\psi_s(x)=s\log(1+e^{\frac{1-x}{s}})$  with small values of $s$ serves as an approximation of the hinge loss $\max(0,1-x)$ and is $1/s$  Lipschitz. This function converges to the Hinge loss when $s\rightarrow 0$ and usual choices of $s$ are usually close to $0$. Therefore,  the Lipschitz constant of the derivative $1/s$ verifies $1/s \gg 1$, and the bounds from \citet{kuzborskij2017fast} become vacuous. Besides, \citet{du2017hypothesis} 
made several assumptions about the true regression function of both the source and target domains. To clarify, by the true regression function, $f$, we refer to the actual model denoted by $Y=f(X)$. However, these assumptions are challenging to empirically confirm due to their reliance on the real source and target distributions, which generally remain unknown. Moreover, the theoretical guarantees achieved depend on several constants, also derived from the true distribution, that makes quantifying the bounds magnitude a complex task.

To our best knowledge, the vast majority of existing theoretical results from the \HTL literature have similar assumptions to those discussed above. However, in this work, our assumptions are flexible:  we only require the differentiability of the loss and a \textit{local} majorant of the derivative, which will make the analysis more flexible and more suited for the usual classification losses. 

To understand the  intuition behind Assumption~\ref{Ass:3} notice that, when $\risk[h_S]\to 0$, $\phi(h_S(X)Y)$ approaches the minimum then $\phi'(h_S(X)Y)$ approaches 0 (in expectation). Thus, the function $\Psi_1$ can be seen as a function that dictates the rate of convergence of the derivative to $0$ as $h_S$ approaches the optimal hypothesis. One must note that the latter assumption is verified for many loss functions, namely any loss satisfying the following inequality $\abs{\phi'(x)}\leq \Psi(\phi(x))$ for some concave loss function $\Psi$. The function $\Psi$ effectively mediates between $\phi$ and $\phi'$. As an example, in the context of Mean Squared Error (MSE) loss, it is straightforwardly observable that $|\phi'(x)|\leq \sqrt{\phi(x)}$. Thus $\phi'$ is directly linked to $\phi(x)$ via the square root function.

\begin{remark}[\textsc{score scaling}]
\textsc{RERM} for regression (cf. Equation \ref{def:rls-regression}) is equivalent to fitting a predictor on the residuals $Y_i-h_S(X_i)$. However, in the classification case, if we follow the standard approach that $h_S: \mathcal{X}\mapsto \mathcal{Y}=\{-1,1\}$ is a binary classifier \cite{mansour2009domain,cortes2015adaptation},  then latter residuals are either $1$ or $0$. Thus, this won't provide enough information for many losses to improve the training. To see this, see the example of the logistic loss and notice that $\phi(1)=\log(1+e^{-1})$ and $\phi(-1)=\log(1+e^1)$. Therefore, in the best case scenario, $\risk[h_S]=\log(1+e^{-1})$, which is far from the minimum (that is zero). To tackle this problem, we suggest taking the score learned on the source, which is more informative, especially when the loss function used to train the algorithm on the source has the same minimum as the loss used to train on the target. Note that one can also think of transforming the score, for example, if $\phi$ 
 is the logistic loss $\phi(x)=\log(1+e^{-x})$ and $h_S\in ]-1,1[$ we can use an increasing transformation function to an interval $]-C,C[$ with $C>>1$ in order to adapt to the target loss which is nearly 0 for large values $x$.  
\end{remark}

\subsection{Pointwise Hypothesis Stability}

To go further than the widely used hypothesis stability, we analyze our \HTL problem through the angle of pointwise hypothesis stability. Results presented in this part will be the cornerstone of those shown in Section~\ref{sec:gener}. To analyze the pointwise hypothesis stability of Algorithm \ref{def:algo-transfer}, we require a direct analogue of Assumption \ref{Ass:3}, involving the data-driven radius provided in Lemma~\ref{lemma:ball}.  
\begin{assumption} \label{Ass:4} 
The function $\phi$ is  differentiable and convex. Furthermore, $\forall i\in[n]$, it holds:
\begin{align*}
	\EE\left[\sup_{\abs{y'},\abs{y}\leq \hat \rho^i_\lambda} \left|\phi'(h_S(X)Y+y')\phi'(h_S(X)Y+y)\right|\right]&\\
	&\hspace{-25mm}\leq \Psi_2\left(\risk\left[h_S\right]\right).
\end{align*}
where $Z=(X,Y)$ is a sample drawn from $P_T$ independent of $\DD^{\backslash i}$  and $\Psi_2$ is a decreasing function  verifying $\Psi_2(0)=0$.
\end{assumption}

Under the latter assumption, the following proposition is obtained in a similar manner to Proposition \ref{prop:stability-main}.

\begin{proposition}\label{prop:pointwise-stability-main}
	Suppose that Assumptions~\ref{Ass:1},~\ref{Ass:2}  and~\ref{Ass:4} are satisfied. Then the algorithm $\alg$ (cf. Equation~\eqref{def:algo-transfer}) is pointwise hypothesis stable with parameter 
	 $$\gamma(n)=\ffrac{\alpha\left(\Psi_2\left(\risk\left[h_S\right]\right)  \wedge \norm{\phi'}_\infty^2\right)}{ n}.$$ 
\end{proposition}
\begin{proof}
The proof is  postponed to the Appendix~\ref{subsec:app:prop3.2}.
\end{proof}
Again, this result shows the benefits of using a good hypothesis on the pointwise hypothesis stability of \RERM. This stability result, combined with that of Proposition~\ref{prop:stability-main}, can be leveraged to propose new convergence results on the generalisation gap and the excess risk of this HTL problem for a wide class of losses, as shown in Section~\ref{sec:gener}. In the sequel, we explicitly compute the functions $\Psi_1$ and $\Psi_2$ for many widely used classification losses.

\subsection{Deriving Constants for Popular Losses}\label{subsec:exhib}

As the results of Propositions~\ref{prop:stability-main} and~\ref{prop:pointwise-stability-main} are general and stated for any losses satisfying Assumptions~\ref{Ass:3} and~\ref{Ass:4}, it is the purpose of this part to investigate our results with widespread machine learning losses. To that end, we first show that these Assumptions are satisfied for the most popular losses. Second, we derive constants involved in these two statistical rates. In particular, we focus on the five following losses:

\begin{itemize}
\item Exponential: $ \phi(x)= e^{-x}$.
\item Logistic: $ \phi(x)= \log\left(1+e^{-x}\right)$.
\item Mean Squared Error: $ \phi(x)= (1-x)^2$.
\item Squared Hinge:  $ \phi(x)= \max(0,1-x)^2$.
\item Softplus: $\phi_s(x)=s\log\left(1+e^{\frac{1-x}{s}}\right)$, for some $s>0$.
\end{itemize}

In the next proposition, we show that most of classical losses verifies  Assumptions~\ref{Ass:3}, \ref{Ass:4} and we detail their associated functions $\Psi_1$ and $\Psi_2$.

\begin{proposition}\label{prop:loss-property}
The exponential, logistic, squared hinge, \MSE and softplus losses satisfy Assumptions~\ref{Ass:3} and \ref{Ass:4} with corresponding functions $\Psi_1$  and $\Psi_2$ listed in Table \ref{table:property-hypothesis}. 
\end{proposition}
\begin{proof}
The proof is  postponed to the Appendix~\ref{subsec:app:prop3.3}.
\end{proof}

This result shows that bounds derived in Propositions~\ref{prop:stability-main} and~\ref{prop:pointwise-stability-main} are therefore valid under mild assumptions. Indeed, our results only require the kernel  and the source hypothesis to be bounded, classical in the HTL framework. Thus, we obtain the first stability result in HTL without limiting assumptions, which remains valid in a practical setting. 

As shown in Table~\ref{table:property-hypothesis}, functions $\Psi_1$ and $\Psi_2$ are linear for the square hinge and the \MSE losses. Besides, for the softplus and logistic losses, we have $\norm{\phi'}_\infty=1$ and their stability parameters capped by $\alpha/n$.
Thus, the impact of an irrelevant source hypothesis $h_S$ with large $\risk[h_S]$  remains negligible on the stability of \RERM  when using these losses. In contrast, for the exponential loss, the functions $\Psi_1$ and $\Psi_2$ are roughly exponential, and the corresponding convergence rate deteriorates quickly as $\risk[h_S]$ increases. This is indeed not surprising since a prediction in the wrong direction ($\sign(h_S(X))\neq Y$) would increase the loss $e^{-h_S(X)Y}$  exponentially fast. In the particular case of the \MSE, we obtain the same stability rate $\mathcal{O}\left(\frac{\alpha\risk[h_S]}{n}\right)$ as in the regression framework \cite{kuzborskij2013}. In the next section, we shall discuss the implications of these stability rates on the  \emph{generalization gap} \cite{hardt2016train,charles18a}, cross-validation schemes and the excess risk of Algorithm \ref{def:algo-transfer}.

\begin{table}
	\begin{center}
		\begin{tabular}{lccr}
		   \hline
		   Loss  & $\Psi_1(x)$ & $\Psi_2(x)$ \\
		   \hline\\[-0.3cm]      
		     Sq. hinge & $8x(4\alpha+1)$ & $8x(4\alpha+1)$ \\
	       \MSE  & $8x(4\alpha+1)$  & $8x(4\alpha+1)$ \\
              Exponential & $C_Sx^2e^{2\alpha x}$ & $M_{S}C_Sxe^{2\alpha x}$\\ 
              Logistic  & $C_Se^{2\alpha x}(e^{\sqrt{x}}-1)^2$&$ C_Se^{2\alpha x}(e^{\sqrt{x}}-1)$\\
              Softplus  & $C_Se^{2\alpha x}(e^{\sqrt{\frac{x}{s}}}-1)^2$ & $C_Se^{2\alpha x}(e^{\sqrt{\frac{x}{s}}}-1)$   \\
              
			\hline
		\end{tabular}

	\end{center}
		\caption{Examples of losses verifying Assumptions \ref{Ass:3}, \ref{Ass:4} and their corresponding functions. The constants $M_{S}$ and $C_{S}$ are given by $M_{S}=\sup_{z\in \mathcal{Z}_T}\loss(h_S,z)$, $C_S=\exp\left\{2 +\frac{2\alpha  M_S}{n}+\frac{4\alpha^2  M_S^2}{n-1}\right\}$.}
            \label{table:property-hypothesis}
\end{table}

\section{Generalisation Guarantees for \HTL with Surrogate Losses} \label{sec:gener}

In this part, we leverage the stability results provided in Section~\ref{sec:stability} in several statistical errors commonly used. 



\subsection{Generalization Gap}\label{subsec:gen}

Here we investigate the accuracy of the algorithm $\alg$ through the generalization gap. Precisely, this gap is defined as the expected error between the empirical risk and the theoretical risk of the algorithm $\alg$: 

\begin{equation*}\label{def:gen-gap}
    \mathcal{E}_{\textrm{gen}}=\abs{\EE\left[\ER\left[\alg(\DD_T)\right]-\risk\left[\alg (\DD_T)\right]\right]}.
\end{equation*}

To discuss the impact of $h_S$ on the generalization gap, it suffices to analyse the stability parameters $\beta(n)$ and $\gamma(n)$. Indeed, $\mathcal{E}_{\textrm{gen}}$ is directly linked to these quantities, as stated in the following theorem.

\begin{theorem}\label{theo:gen-gap-UB}
Suppose that $\alg$ has a hypothesis stability $\beta(n)$ and a pointwise hypothesis  stability $\gamma(n)$. Then, it holds:
$$\mathcal{E}_{\textrm{gen}}\leq \beta(n)+\gamma(n).$$
Furthermore, suppose that Assumptions~\ref{Ass:1},~\ref{Ass:2},~\ref{Ass:3} and~\ref{Ass:4} are satisfied. Thus,  $\beta(n)$ and $\gamma(n)$ are given by Propositions~\ref{prop:stability-main} and~\ref{prop:pointwise-stability-main} and  the generalization gap of $\alg$ (cf. Equation~\eqref{def:algo-transfer}) is upper-bounded as:
    $$\mathcal{E}_{\textrm{gen}}\leq \alpha\ffrac{\left(\Psi_1\left(\risk\left[h_S\right]\right)+  \Psi_2\left(\risk\left[h_S\right]\right)\right)\wedge \left(2\norm{\phi'}_\infty^2\right)}{ n}.$$
\end{theorem}


\begin{proof}
The proof is  postponed to the Appendix~\ref{subsec:app:gen-gap-UB}.
\end{proof}

When the source hypothesis is relevant, the risk  $\risk[h_S]$ is close to zero so that $e^{\risk[h_S]}-1\approx \risk[h_S] $ and $e^{\alpha\risk[h_S]}\approx 1 $.  Equipped with Table~\ref{table:property-hypothesis}, this theorem yields the following upper bounds for $\mathcal{E}_{\textrm{gen}}$:
\begin{itemize}
    \item \MSE, Sq. hinge: $\mathcal{E}_{\textrm{gen}}=\mathcal{O}\left(\frac{\alpha\risk[h_S]}{n}\right)$.
    \item Logistic:  $\mathcal{E}_{\textrm{gen}}=\mathcal{O}\left(\alpha\frac{\sqrt{\risk[h_S]}\wedge 2}{n}\right)$.
    \item Softplus:
    $\mathcal{E}_{\textrm{gen}}=\mathcal{O}\left(\alpha\frac{\left(\sqrt{\risk[h_S]/s}\right)\wedge 2}{n}\right)$.
    \item Exponential: $\mathcal{E}_{\textrm{gen}}=\mathcal{O}\left(\frac{\alpha M_{S}\risk[h_S]}{n}\right)$. 
\end{itemize}

Thus, if $\risk[h_S]$ is small, the exponential, the squared hinge and the \MSE losses have the fastest generalization gap rate. Therefore, our analysis suggests that the user should privilege using the latter losses if one disposes of a good hypothesis $h_S$. 


\noindent \textbf{Negative learning}
 The phenomenon of negative transfer occurs when the hypothesis $h_S$ learned from the source domain has a detrimental effect on the target learner. In such a case, training without using $h_S$ on the target domain would yield a better learner. We refer the reader to \citet{weiss2016survey} and \citet{wang2019characterizing} for further details about this topic. For the softplus and the logistic losses, the generalization gap remains bounded by $\mathcal{O}(\alpha/n)$ even if $\risk[h_S]\to \infty$. As a consequence, Algorithm~\ref{def:algo-transfer} with the sofplus and logistic losses is robust to negative learning since the generalization gap still achieves the same rate of convergence $\mathcal{O}(\alpha/n)$ as a standard RERM algorithm with no source information  \ie $h_S=0$ (see \eg \citealp{Zhang2004,wibisono2009sufficient}). Finally, we must highlight that one should avoid using the exponential loss when the source and target domains are unrelated due to the presence of the term $e^{\alpha\risk[h_S]}$ in the corresponding upper bound.

\begin{remark}[\textsc{cross validation procedures}]\label{rk:crossval}
The notion of stability has many attractive qualities. In particular, it yields complexity-free bounds for cross-validation methods. (see \eg \citealp{bousquet2002stability,kumar2013near,celisse2018theoretical}). For example, one can easily show that $$\EE\left[\left|\rloo\left[\alg\left(\DD_T\right)\right]-\risk\left[\alg\left(\DD_T\right)\right]\right|\right] \leq \beta(n).$$
    Proposition \ref{prop:stability-main} shows that the quality of risk estimation with \loo depends directly on the quality of the source predictor $h_S$. Note that the same conclusion holds for model selection with \loo cross-validation: Given a family of source hypotheses, the quality of the model selection procedure depends directly on the quality of the provided learners independently of the complexity of $\mathcal{H}_T$. Besides, using the same proof techniques, we can show that Algorithm \ref{def:algo-transfer} is $L_2$ stable with stability parameter depending on $\Psi\left(\risk\left[h_S\right]\right)$. $L_2$ stability is similar to hypothesis stability, where the $L_1$ moment is replaced by the $L_2$ moment in Definition \ref{def:hypothesis-stable}. The latter notion allows obtaining theoretical guarantees regarding the \Kfold\ and the \loo schemes. It also derives asymptotic confidence intervals for cross-validation procedures in risk estimation and model selection \cite{bayle2020cross,austern2020asymptotics}. In our particular case, Proposition \ref{prop:stability-main} implies that the tightness of the confidence intervals of cross-validation methods depends only on the quality of $h_S$.
\end{remark}

\subsection{Excess Risk}
In this section we  analyse the excess risk of Algorithm \ref{def:algo-transfer} defined as:
$$ \mathcal{E}_{\textrm{ex}}=\EE\left[\risk\left[\alg\right] -\risk\left[h^*+h_S\right]\right],$$

where $h^*=\argmin_{h\in \mathcal{H}}\risk\left[h_S+h\right]$. To this end, we start by showing  that $\mathcal{E}_{\textrm{ex}}$ depends on the upper bounds of the \textit{ (pointwise) hypothesis stability}
and the regularization parameter $\lambda$. Further, we derive precise finite-sample rates for the surrogate losses introduced in Section~\ref{subsec:exhib}.
\begin{theorem}\label{theo:custom-stability-surrogate}
	Suppose that $\norm{h^*}_k < \infty$. Then, the excess risk of algorithm \ref{def:algo-transfer} verifies,
	$$ \mathcal{E}_{\textrm{ex}}\leq \gamma(n)+\beta(n)+\lambda\norm{h^*}_k^2.$$

Making $\lambda$ varying with the sample size $n$, we obtain various consistent bounds for different losses. In the sequel, we assume that $\kappa\leq 1$ and $M_S\leq 1$ to avoid notional burden. When $\phi$ is either the \MSE or the squared hinge  and  $\lambda=\sqrt{\ffrac{\risk[h_S]}{\sqrt{n}}}$, it holds: 
$$\mathcal{E}_{\textrm{ex}}\leq \mathcal{O}\left( \sqrt{\ffrac{\risk[h_S]}{\sqrt{n}}} \right).$$

Furthermore, if $\phi$ is the exponential loss and  $n\geq \frac{M_S^2\ln(n)^2} {\risk[h_S]}$, picking  $\lambda=4\ffrac{\sqrt{\risk[h_S]}\wedge 1}{\ln(n)}$ yields:
$$\mathcal{E}_{\textrm{ex}}\leq \mathcal{O}\left(  \ffrac{\sqrt{\risk[h_S]}\wedge 1}{\ln(n)} \right),$$
otherwise picking $\lambda=\frac{\ln(n)^2}{\sqrt{n}}$  gives:
$$\mathcal{E}_{\textrm{ex}}\leq \mathcal{O}\left(  \frac{\ln(n)^2}{\sqrt{n}} \right).$$
Suppose that the function $\phi$ is the logistic loss or the softplus. Then the choice $\lambda=\ffrac{1}{\sqrt{n}}$ yields:
$$\mathcal{E}_{\textrm{ex}}\leq \mathcal{O}\left(  \frac{1}{\sqrt{n}} \right).$$
\end{theorem}

In particular, Theorem~\ref{theo:custom-stability-surrogate} yields the consistency of \RERM. Furthermore,  the Remark~\ref{rk:crossval} regarding the generalization gap still holds for the excess risk. First, when $\risk[h_S]$ is small, Algorithm \ref{def:algo-training-transfer} with \MSE or squared hinge would have the fastest convergence rate.
Second, when  $\risk[h_S]$ is large compared to the sample size $n$, then the safest option is to use the logistic or the softplus losses with $\lambda=\frac{1}{\sqrt{n}}$. Note that, if $\risk[h_S]$ is small  an improved convergence rate \big($1/\sqrt{-n\ln\left(\risk[h_s]\right)}$\big)  can be achieved  for the latter losses  (see Appendix \ref{subsec:excess-risk}  for further details).
Finally, Algorithm~\ref{def:algo-transfer} with the exponential loss is likely to suffer from negative learning. Indeed, if $\risk[h_S]$ is large, one needs a large amount of data to ensure the non-triviality of the rate $\risk[h_S]/\ln(n)$. It is worth noting that the rate of convergence with the exponential loss is naturally logarithmic even without a source hypothesis; see, for instance, Corollary 4.1 and Theorem 4.4 in \citet{Zhang2004}. To conclude, using a good source hypothesis improves convergence rates of \RERM  compared to those derived without transfer \cite{Zhang2004}. 


\begin{remark}[\textsc{on the universal consistency}]
If we assume that the kernel $k$ is non-polynomial, $h_S$ is continuous and the distribution of $X\in \mathcal{X}_T$ is \emph{regular} (see \eg Definition 4.2 in \citealp{Zhang2004}). Then, one can use any \emph{universal approximation theorem}  (see for instance Theorem 4.1 in \citealp{Zhang2004}) to obtain 
$$h^*=\argmin_{h \in \mathcal{H}}\risk\left[h_S+h\right]=\argmin_{h \in \mathcal{L}(\mathcal{X}_T, \mathbb{R})}\risk\left[h_S+h\right],$$
where $\mathcal{L}(\mathcal{X}_T, \mathbb{R})$ is the space of real-valued functions defined on $\mathcal{X}_T$. The universal consistency of $\alg$ follows immediately from Theorem~\ref{theo:custom-stability-surrogate}. Further, all the losses presented in this paper are \emph{classification calibrated} \citep{bartlett2006convexity} meaning that:
$$\argmin_{h \in \mathcal{L}(\mathcal{X}_T, \mathbb{R})}\risk\left[h\right]=\argmin_{h \in \mathcal{L}(\mathcal{X}_T, \mathbb{R})}\risk^{0\text-1}\left[h\right], $$

where $\risk^{0\text-1}\left[h\right]=P_T(\sign\left(h(X)\right)\neq Y)$ is the usual classification accuracy. Thus, minimizing the excess risk would likely yield a classifier with good accuracy.
\end{remark}



\section{Numerical experiments}
We illustrate our analysis by providing some results using simulated data that aim to underscore the robustness of each loss to negative learning scenarios. The experiment is conducted as follows. A source domain is considered with random variables $(X_S,Y_S)\in \mathbb{R}^2\times \{-1,1\}$, where the positive and negative classes are respectively drawn from two multivariate $t$-distributions  $\mathcal{T}((r,0),3I_2,2.5)$ and $\mathcal{T}((-r,0),3I_2,2.5)$. We train a linear classifier $h_S$ on a source dataset of size $10000$ using the \textrm{SVM} algorithm.

To emphasize the impact of negative learning on each loss, we generate a smaller target dataset of size $100$. The distributions for positive and negative classes are given by $\mathcal{T}(((r+d)cos(\theta),(r+d)sin(\theta)),I_2,2.5)$ and $\mathcal{T}((-(r+d)cos(\theta),-(r+d)sin(\theta)),I_2,2.5)$, respectively. For different values of $\theta$, the target risk $\risk\left[\hat h+h_s\right]$ of the analyzed RERM  algorithm (with $\lambda=1$) trained on the small size dataset  is estimated using a test set of size $10000$.

It is important to note that when $\theta=0$, it corresponds to the scenario of positive learning since the decision boundaries of both domains are similar. On the other hand, the case where $\theta=\pi$ corresponds to negative learning since the true decision functions of the source and the target domain are pointing to opposite directions.

Figure \ref{fig:negative_learning} presents the median true risk of the HTL algorithm (cf. Equation \ref{def:algo-training-transfer}) as a function of $\theta$ for $(r,d)=(5,5)$ computed over $1000$ simulations. The parameter $s$ of the softplus loss is set to $0.1$. Consistent with our theoretical analysis, the softplus and logistic functions exhibit significant robustness to negative transfer. 
 \begin{figure}
     \centering
     \includegraphics[scale=0.4]{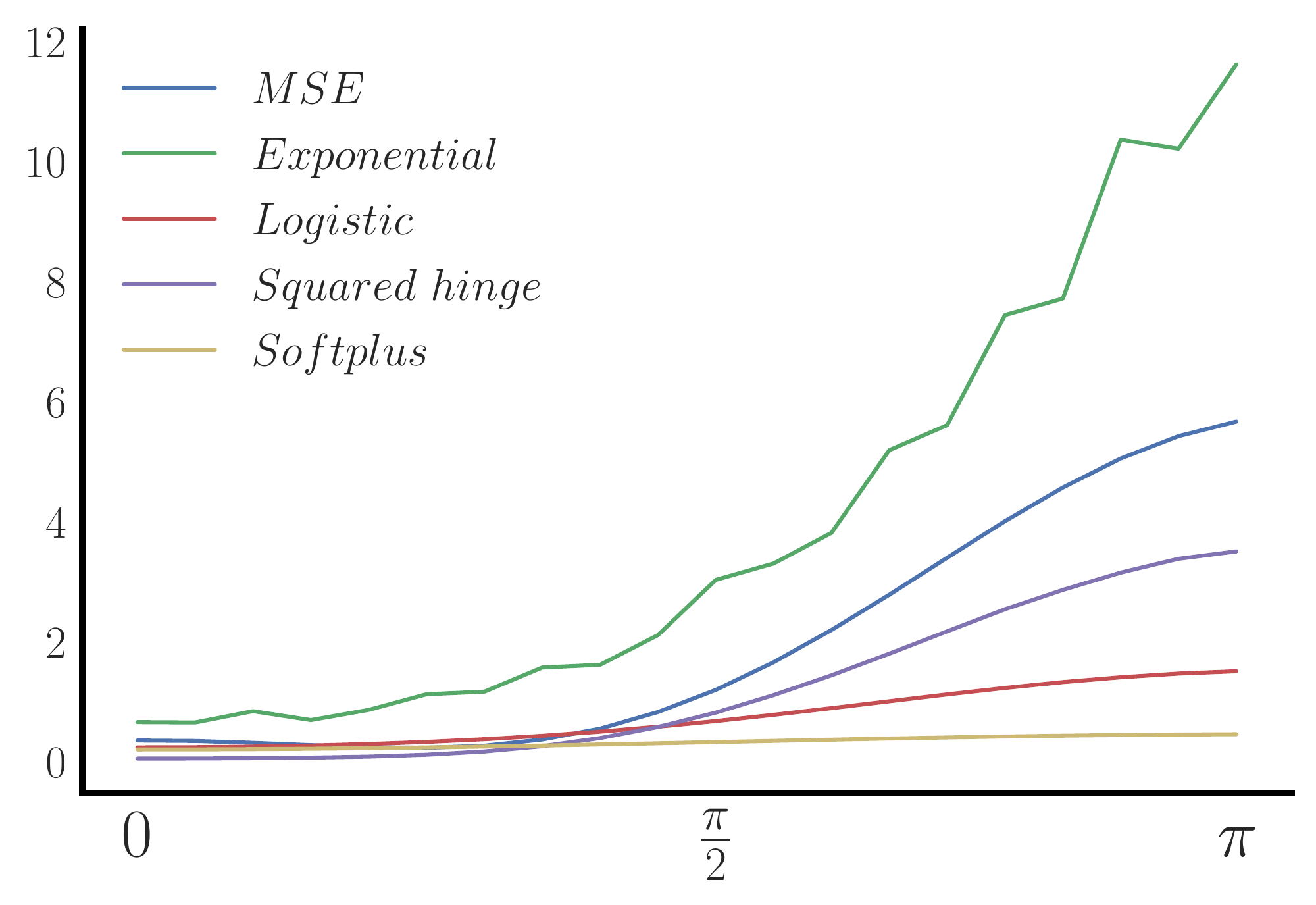}
     \caption{ Target risk of Algorithm \ref{def:algo-transfer} as a function of $\theta$.}
     \label{fig:negative_learning}
 \end{figure}

\section{Conclusion}

In this paper, we study hypothesis transfer learning through the angle of Algorithmic Stability. Following the work of \citet{kuzborskij2013}, where hypothesis stability is shown for the \MSE in the regression setting, we derive similar hypothesis stability rates in classification with general losses under slight assumptions. Furthermore, we show that our assumptions are satisfied for the most popular machine learning losses, making our work valuable for practitioners. Moreover, we leverage our stability results to provide finite-sample analysis on the generalization gap and the excess risk. We show that \HTL framework is efficient and explicit (fast) rates for these popular losses. Our theoretical analysis will help practitioners better understand the benefits of \HTL and give insight into the loss choices.

The proposed work is general and may fit with many other domains.
Future work may involve our analysis for different Machine Learning tasks where transfer learning procedures can be beneficial such as robust learning \citep{shafahi2020adversarially,laforgue2021generalization,staerman2021ot}, anomaly detection \citep{andrews2016transfer,chandola2009anomaly, staerman2020area,staerman2022functional}, speech \cite{campi2021machine,campi2023ataxic},
automatic language generation 
\citep{staerman2021pseudo,golovanov2019large},  knowledge distillation \citep{cho2019efficacy}, events-based modelling \cite{staerman2022fadin}, fairness \cite{colombo2022learning}
or general neural-networks based tasks \citep{colombo2022beyond,picothalfspace,darrin2023unsupervised}.
\bibliography{biblio_crossval}

\newpage
\appendix
\onecolumn
\section{Preliminary Results}\label{sec:intermediate}
In this section, we show some useful technical lemmas used in the subsequent proofs.

\begin{lemma}\label{lemma:expectation-mult-control}
	Suppose that $X,Y,Z$ are three mutually independent random variables such that $E(X)=E(Y)$. Then it holds:
	$$\EE\left[(X+Z)(Y+Z)\right]\leq 2\left(\EE\left[X\right]^2+\EE[Z^2]\right).$$
\end{lemma}
\begin{proof}
Since $X,Y,Z$  are mutually independent one has the following identities,
\begin{align*}
\EE\left[(X+Z)(Y+Z)\right]
&=\EE[X]\EE[Y]+\EE[X]\EE[Z]+\EE[Z]\EE[Y]+\EE[Z^2]\\
&=\EE[X]^2+2\EE[X]\EE[Z]+\EE[Z^2].
\end{align*}
Now, noticing that $\left(E[Z]^2\leq E[Z^2]\right)$ we get:
\begin{align*}
    \EE[X]^2+2\EE[X]\EE[Z]+\EE[Z^2] &\leq 2\EE[X]^2+\EE[Z]^2+\EE[Z^2] 
    \\& \leq 2\left(\EE\left[X\right]^2+\EE[Z^2]\right),
\end{align*}

which is the desired result.
\end{proof}
In the sequel, we shall provide an upper bound for the exponential of $\hat \tau^i_\lambda$ defined as:

\begin{equation}\label{def:ball-prime}
\hat \tau^i_\lambda=\sqrt{\alpha \left(\ER^{\backslash i}[h_S]+\frac{M_S}{n}\right)},
\end{equation}

with $M_S=\sup_{z\in\mathcal{Z}_T}\loss(h_S,z)$ and 
\begin{equation}\label{def:ER-pop}
  \ER^{\backslash i}[h]= \frac{1}{n-1}\sum_{j\neq i}\loss(h,Z_j),
\end{equation}
the training error of a hypothesis $h$ with the $i$'th datum removed. The quantity  $\hat \tau^i_\lambda$  will serve as an upper bound of $\hat \rho^i_\lambda=\max(\hat r_\lambda, \hat r^i_\lambda)$ independent of the observation $Z_i\in \DD_T$. Indeed,  by definition:

$$\hat \tau^i_\lambda \geq  \hat r^i_\lambda=\sqrt{\alpha \left(\ER^{\backslash i}[h_S]\right)}.$$
Moreover, it holds:

$$ \ER\left[h\right]\leq \ER^{\backslash i}\left[h\right]+\frac{\loss(h,Z_i)}{n} \leq \ER^{\backslash i}\left[h\right]+\frac{M_S}{n},$$
so that $\hat \tau^i_\lambda  \geq \hat r^i_\lambda$. Thus, we have $\hat \tau^i_\lambda  \geq \hat \rho^i_\lambda$.

\begin{lemma}\label{lemma:exponential-sum-control}
	Let $W_1,W_2,\dots, W_n$ be a sequence  of \iid random variables bounded by $C>0$. Then one has
	$$\EE\left[e^{\hat \mu}\right] \leq  e^{\mu+\frac{C^2 }{n}},$$
	where $\mu=\EE\left[W_1\right]$ and $\hat \mu= \frac{1}{n} \sum_{i=1}^{n}W_i$.
\end{lemma}
\begin{proof}
	The proof follows in two steps. First, we apply Hoeffding's inequality to obtain:
	$$\PP\left(\abs{\hat \mu - \mu} \geq t\right) \leq e^{\frac{-nt^2}{C^2}}.$$
	Second, applying Theorem 2.5.2 in \citet{vershynin_2018} yields:
		$$\EE\left[e^{\hat \mu-\mu}\right] \leq e^{\frac{C^2 }{n}},$$
which leads to the desired result.
\end{proof}

\begin{lemma}\label{lemma:exponential-radius-control}
For all $i\in [n]$ and $p \in \mathbb{N}$, the quantity $e^{\hat \tau^i_\lambda}$ verifies:
$$\EE\left[e^{p \hat \tau^i_\lambda}\right]\leq e^{p +\frac{\alpha p M_S}{n}+\frac{\alpha^2 p^2 M_S^2}{n-1}}e^{p \alpha\risk[g]}.$$
\end{lemma}

\begin{proof}
    First, using the fact that $\sqrt{x}\leq x+1$, one has:
    $$e^{\hat \tau^i_\lambda}\leq e^{p \alpha \ER^{\backslash i}\left[h_S\right]+\alpha\frac{p M_S }{n}+p}.$$
    Since $p \alpha\ER^{\backslash i}\left[h_S\right]=\ffrac{1}{n-1}\sum_{i\neq j} p \alpha\loss(h,Z_i)$, applying Lemma~\ref{lemma:exponential-sum-control} with $W_i=p \alpha\loss(h,Z_i)$ and $C=p \alpha M_{S}$ yields the desired result.
    
\end{proof}
To prove Propositions~\ref{prop:stability-main} and~\ref{prop:pointwise-stability-main}, we extend  Theorem 4.3 in \citet{Zhang2004}, that gives an upper bound for standard \RERM to the \HTL framework. This extension leads to the next lemma.

 \begin{lemma}\label{theo:main-lemma-stability}
	The leave one out deviations of the algorithm $\alg$ (cf. Equation \eqref{def:algo-transfer}) verifies: 
	$$\norm{\alg\left(\DD_T\right)-\alg\left(\DD_T^{\backslash i}\right)}_{k} \leq \frac{k\left(X_i,X_i\right)^{1 / 2}\left|\phi^{\prime}\left(\alg\left(\DD_T,X_i\right)Y_{i} \right)\right| }{\lambda n}. $$
\end{lemma}

\begin{proof}
    Since $\phi$ is convex, the Bregman divergence of $\phi$ is non negative. More precisely,
 
        $$d_\phi(x,y)=\phi(x)-\phi(y)-(x-y)\phi'(y)\geq 0,$$  
so that, for any  $Z_i=(X_i,Y_i)\in \DD_T$ one has:
	\begin{align*}
		\loss\left(\alg\left(\DD_T^{\backslash i}\right),Z_i\right)-d_\phi\left(\alg\left(\DD_T^{\backslash i},X_i\right)Y_i,         \alg\left(\DD_T,X_i\right)Y_i\right)\leq \loss\left(\alg\left(\DD_T^{\backslash i}\right),Z_i\right),
	\end{align*}
	
where $\alg\left(\DD_T,X_i\right)$ is the prediction of the input $X_i$ by the algorithm $\alg$.	Also, the term on the left side in the above inequality can be written as follows:
	\begin{align*}
		\loss\left(\alg\left(\DD_T^{\backslash i}\right),Z_i\right)-&d_\phi\left(\alg\left(\DD_T^{\backslash i},X_i\right)Y_i,         \alg\left(\DD_T,X_i\right)Y_i\right)=\loss\left(\alg(\DD_T),Z_i\right)\\+
		&\phi'\left(\alg(\DD_T,X_i)Y_i\right)\left(\alg\left(\DD_T^{\backslash i},X_i\right)-\alg\left(\DD_T,X_i\right)\right)Y_i,
	\end{align*}
	so that:
	$$\loss\left(\alg(\DD_T),Z_i\right)+ \phi'\left(\alg(\DD_T,X_i)Y_i\right)\left(\alg\left(\DD_T^{\backslash i},X_i\right)-\alg\left(\DD_T,X_i\right)\right)Y_i\leq \loss\left(\alg(\DD_T^{\backslash i}),Z_i\right). $$
	Thus, we get:
	\begin{equation}\label{ineq:key-stability-ineq}
	\ER^{\backslash i}\left[\alg\left(\DD_T\right)\right]+ S_i \leq \ER^{\backslash i}\left[\alg\left(\DD_T^{\backslash i}\right)\right],
	\end{equation}
	
	where $\ER^{\backslash i}$ is defined previously in Equation~\eqref{def:ER-pop} and $S_i=\ffrac{1}{n}\sum_{j\neq i}\phi'\left(\alg(\DD_T,X_j)Y_j\right)\left(\alg\left(\DD_T^{\backslash i},X_j\right)-\alg(\DD_T,X_j)\right).$
	
	Let $\hat h^{\backslash i}$ denote the solution of the optimization problem \ref{def:algo-training-transfer} with the $i$'th datum removed. One gets by definition of $\alg$ (cf. Equation~\eqref{def:algo-transfer}),
	$$
	\ER^{\backslash i}\left[\alg\left(\DD_T^{\backslash i}\right)\right] +\lambda\lVert \hat h^{\backslash i} \rVert_{k}^2 \leq \ER^{\backslash i}\left[\alg\left(\DD_T\right)\right]+ \lambda\lVert \hat h \rVert_{k}^2.
	$$
	Using \eqref{ineq:key-stability-ineq}, it yields:
	\begin{align*}
            S_i &\leq\lambda\left(\norm{\hat h}_k^2-\norm{\hat h^{\backslash i}}_k^2\right) \nonumber\\
		 &\leq -\lambda\lVert \hat h-\hat h^{\backslash i}                 \rVert_{k}^2-2\lambda\langle \hat h,\hat h^{\backslash i}-      
             \hat h\rangle,
	\end{align*}
where the second line follows from $\norm{x}-\norm{y}=\norm{x-y}^2+2\langle x-y, y \rangle$. Reverting the inequality leads to:

\begin{align}\label{ineq:norm-control}
    		\lambda\lVert \hat h-\hat h^{\backslash i} \rVert_{k}^2 
		&\leq -\ffrac{1}{n_T}\sum_{j\in T^{\backslash 
            i}}\phi'\left(\alg(\DD_T,X_j)Y_j\right)\langle \hat h^{\backslash i}-\hat h,k\left(X_i,\cdot\right)
		\rangle -2\lambda\langle \hat h,\hat h^{\backslash i}-\hat h\rangle 
            \nonumber
		\\ 
		& \leq \bigg\lVert \ffrac{1}{n_T}\sum_{j\in T^{\backslash i}}\phi'\left(\alg(\DD_T,X_j)Y_j\right)k\left(X_i,\cdot\right)
		+2\lambda g\bigg\rVert_k   \lVert \hat h^{\backslash i} - \hat h \rVert_k.
\end{align}

	 The last inequalities hold because of the definition of $S_i$:
 \begin{align*}
 S_i&=\ffrac{1}{n}\sum_{j\neq i}\phi'\left(\alg(\DD_T,X_j)Y_j\right)\left(\alg\left(\DD_T^{\backslash i},X_j\right)-\alg(\DD_T,X_j)\right)\\
 &=\ffrac{1}{n}\sum_{j\neq i}\phi'\left(\alg(\DD_T,X_j)Y_j\right)\left(\hat h^{\backslash i}\left(X_j\right)-\hat h\left(X_j\right)\right)\\
 &=\ffrac{1}{n}\sum_{j\neq i}\phi'\left(\alg(\DD_T,X_j)Y_j\right)\langle \hat h^{\backslash i}-\hat h,k(X_j,\cdot)\rangle.
 \end{align*}
	
	On the other hand, since $A(\DD_T,X_j)=h_S(X_j)+\langle \hat h,k(X_j,\cdot)\rangle$  and by Theorem 3.1.20 in \citet{nesterov2018lectures}, we know that the following optimality condition holds:
	$$\ffrac{1}{n}\sum_{j=1}^{n} \phi'\left(\alg(\DD_T,X_j)Y_j\right)k\left(X_j,\cdot\right)+2\lambda \hat h=0.  $$
	Therefore Inequality \eqref{ineq:norm-control} becomes:
	
	$$ 	\lambda\lVert \hat h-\hat h^{\backslash i} \rVert^2  \leq \bigg\lVert \ffrac{1}{n}\phi'\left(\alg(T,X_i)Y_i\right)\bigg\rVert_k  \lVert k(X_i,\cdot) \rVert_k  \lVert \hat h^{\backslash i} - \hat h \rVert_k.  $$
	
	it remains to remind that $\lVert k\left(X_i,\cdot\right) \rVert^2=k\left(X_i,X_i\right)$ and $\norm{\alg\left(\DD_T\right)-\alg\left(\DD_T^{\backslash i}\right)}_{k}=\norm{ \hat h^{\backslash i} - \hat h}_k$ to complete the proof.
\end{proof}

Before highlighting the link between Algorithm \ref{def:algo-transfer} with \textsc{RLS}, let's remind a useful lemma (representer theorem) that allows simplifying the optimization problem~\ref{def:algo-training-transfer} in practice. 
\begin{lemma}\label{lemma:finite-dim-sol}
	The learning rule $\hat h$ (cf. Equation \ref{def:algo-training-transfer}) lies in the linear span in $\mathcal{H}$ of the vectors $\left(k\left(X_i,\cdot\right)\right)_{1\leq i\leq n}$, \ie
	\[\hat h\in \mathcal{H}_{\DD}, \]
	with $\mathcal{H}_{\DD}=\left\{\sum_{1}^{n}\alpha_ik\left(X_i,\cdot\right) \mid \alpha_1,\dots,\alpha_n \in \mathbb{R}\right\}$.
\end{lemma}

\begin{proof}
	Since $\mathcal{H}_{\DD}$ is a finite dimensionnal subspace of $\mathcal{H}$,  any $h\in\mathcal{H}$ can be decomposed as:
	$$h=h_\DD+h^{\perp},$$
	with  $h_\DD\in \mathcal{H}_{\DD}$   and $h^{\perp} \perp \mathcal{H}_{\DD}$. Furthermore using the fact that $h(x)=\langle h , k(x,\cdot) \rangle_{k} $, for all $i\in [n]$, one obtains:
	$$h(X_i)=\langle h , k(X_i,\cdot)Y_i \rangle =\langle h_\DD , k(X_i,\cdot)Y_i \rangle=h_\DD(X_i)Y_i.$$
	Thus, for any $Z_i\in\DD_T$, it holds:
	$$\loss(h+h_S,Z_i)=\phi\left(\left(h(X_i)+h_S(X_i)\right)Y_i \right)=\phi\left(\left(h_\DD\left(X_i\right)+h_S\left(X_i\right)\right)Y_i \right)=\loss(h_\DD+h_S,Z_i),$$
	which gives 
	\begin{equation*}
	\widehat{\mathcal{R}}(h+h_S)=\widehat{\mathcal{R}}(h_\DD+h_S).
	\end{equation*}
	On the other hand, by the Pythagorean theorem,
	\begin{equation*}
	\lVert h_\DD\rVert_k^2 \leq \lVert h\rVert_k^2,
	\end{equation*}
	and
        \begin{equation*}
	\widehat{\mathcal{R}}(h+h_S)+\lambda \lVert h_\DD\rVert_k^2\leq \widehat{\mathcal{R}}(h_\DD+h_S)+\lambda \lVert h_\DD\rVert_k^2.
	\end{equation*}
	Thus, the solution of the minimization problem \ref{def:algo-training-transfer}  must lie in $\mathcal{H}_\DD$.
\end{proof}

\subsection{Link with Least Squares with Biased Regularisation}\label{subsec:app:RLS}

To begin, it is a well know fact that, when the kernel $k$ is linear then the \textsc{RKHS} space consists of the set of linear classifiers:
$$ \mathcal{H}=\left\{h(x)=u^\top x \mid u\in \rset^d\right\}.$$
In this case, the solution of the optimization problem  with the mean square loss $\loss(h,Z)=(1-h(X)Y)^2$, writes as $\hat h =\hat u^\top x$ with 
 \begin{align*}
    \hat{u}&=\argmin_{u \in \rset^d}\frac{1}{n}\sum_{i=1}^{n}\left(u^\top X_iY_i +h_S(X_i)Y_i-1\right)^2+\lambda\lVert u\rVert_2^2\\
    &=\argmin_{u \in \rset^d}\frac{1}{n}\sum_{i=1}^{n}Y_i^2\left(u^\top X_i +h_S(X_i)-\frac{1}{Y_i}\right)^2+\lambda\lVert u\rVert_2^2\\
   &=\argmin_{u \in \rset^d}\frac{1}{n}\sum_{i=1}^{n}\left(u^\top X_i +h_S(X_i)-Y_i\right)^2+\lambda\lVert u\rVert_2^2,
 \end{align*}
where the last inequality follows from the facts that $Y_i^2=1$ and  $\frac{1}{Y_i}=Y_i$. Furthermore, if $h_S(x)=v^T x$ for some $v\in \rset^d$ one has:
 \begin{align*}
    \hat{u}&=\argmin_{u \in \rset^d}\frac{1}{n}\sum_{i=1}^{n}\left((u+v)^\top X_i -Y_i\right)^2+\lambda\lVert u\rVert_2^2\\
    &=\argmin_{u \in \rset^d}\frac{1}{n}\sum_{i=1}^{n}\left(u^\top X_i -Y_i\right)^2+\lambda\lVert u-v\rVert_2^2.
 \end{align*}
This is the original form of biased regularisation algorithms.

\section{Technical Proofs of the Main Results}\label{sec:main-proofs}
Before starting the proof of our main results, we remind two  properties of  \textrm{RKHS} spaces that are:
$$\forall x,y \in \mathcal{X}_T\;,\quad \langle k(y,\cdot),k(x,\cdot) \rangle=k(x,y), $$
and 
$$\forall h\in \mathcal{H}\;,\;\forall x \in \mathcal{X}_T\;,\quad h(x)=\langle h,k(x,\cdot) \rangle.$$

Under Assumption~\ref{Ass:2}, using Cauchy Schwartz-inequality yields:

$$\forall h \in \mathcal{H}\;,\quad \norm{h}_\infty\leq  \sqrt{\kappa}\norm{h}_k.$$

\subsection{Proof of Lemma \ref{lemma:ball}}\label{subsec:app:lemma3.1}

This lemma follows from our assumptions and a simple fact. Indeed, notice that by definition of $\hat h$
	$$ \ER(\hat h+h_S)+\lambda\lVert \hat h \rVert^2\leq \ER(\mathbf{0}+h_S).$$	
 Furthermore, $\ER(h_S+\hat h)$ is non-negative since $\phi$ is non-negative which concludes the proof. 
 
\subsection{Proof of Proposition~\ref{prop:stability-main} }\label{subsec:app:prop3.1}

Let $Z=(X,Y)\in \mathcal{Z}_T$ and remind that, by definition of $\alg$, one has: 
\begin{align*}
\left|\loss\left(\alg\left(\DD_T\right),Z\right)-\loss\left(\alg\left(\DD_T^{\backslash i}\right),Z\right)\right|=\left|\phi\left(\left(\hat h(X)+h_S(X)\right)Y\right)-\phi\left(\left(\hat h^{\backslash i}(X)+h_S(X)\right)Y\right)\right|,
\end{align*}
where $\hat h$ is the solution of the optimization problem \ref{def:algo-training-transfer}. Moreover, since $\phi$ is differentiable, one can apply the mean value theory  to obtain:
\begin{align*}
\left|\loss\left(\alg\left(\DD_T\right),Z\right)-\loss\left(\alg\left(\DD_T^{\backslash i}\right),Z\right)\right|&=\abs{\phi^\prime\left(\left(y_\DD+h_S(X)\right)Y\right)}\abs{\hat h(X)-\hat h^{\backslash i}(X)}\\
&\leq \sqrt{\kappa}\abs{\phi^\prime\left(\left(y_\DD+h_S(X)\right)Y\right)}\norm{\hat h-\hat h^{\backslash i}}_k\\
&=\sqrt{\kappa}\abs{\phi^\prime\left(\left(y_\DD+h_S(X)\right)Y\right)}\norm{\alg(\DD_T)-\alg\left(\DD_T^{\backslash i}\right)}_k,
\end{align*}
for some $\abs{y_\DD}\leq \max\left(\hat h(X),\hat h^{\backslash i}(X)\right)$. By Lemma  \ref{lemma:ball}, we have  $\abs{y_\DD}\leq\hat \rho^i_\lambda=\max\left(\hat r_\lambda,\hat r^{\backslash i}_\lambda\right)$. Now, Using Theorem~\ref{theo:main-lemma-stability} with Assumption~\ref{Ass:2} yields:
\begin{equation}\label{eq:exact-stability}
	\left|\loss\left(\alg\left(\DD_T\right),Z\right)-\loss\left(\alg\left(\DD_T^{\backslash i}\right),Z\right)\right|\leq\kappa\ffrac{\abs{\phi^\prime\left(\left(y_\DD+h_S(X)\right)Y\right)\phi^\prime\left(\left(\hat h(X_i)+h_S(X_i)\right)Y_i\right)}}{\lambda n},
\end{equation} 

which gives using the fact that $\norm{\hat h}_\infty \leq \hat r_\lambda \leq \hat \rho_\lambda$:
\begin{equation}\label{eq:prepare-last-step}
\left|\loss\left(\alg\left(\DD_T\right),Z\right)-\loss\left(\alg\left(\DD_T^{\backslash i}\right),Z\right)\right|\leq\sup_{\abs{y},\abs{y'}\leq \hat \rho^i_\lambda }\ffrac{ \alpha\left|\phi'(h_S(X_i)Y_i+y)\phi'(h_S(X)Y+y')\right|}{\ n},
\end{equation}

with $\alpha=\ffrac{\kappa}{\lambda}$. Now, by taking the expectation  and using the fact that $\phi$ verifies assumption~\ref{Ass:3},  Inequality~\eqref{eq:prepare-last-step} becomes:
\begin{equation*}
\EE\left[\left|\loss\left(\alg\left(\DD_T\right),Z\right)-\loss\left(\alg\left(\DD_T^{\backslash i}\right),Z\right)\right|\right]\leq \alpha\ffrac{\Psi_1(\risk\left[h_S\right])}{n}.
\end{equation*} 

Besides, notice that by Equation~\eqref{eq:exact-stability}, 
\begin{equation*}
\forall \DD_T \in \mathcal{Z}_T^n,\;\forall Z \in \mathcal{Z}_T\;,\;\left|\loss\left(\alg\left(\DD_T\right),Z\right)-\loss\left(\alg\left(\DD_T^{\backslash i}\right),Z\right)\right|\leq\alpha\ffrac{ \norm{\phi'}_\infty^2}{ n}.
\end{equation*} 
It remains to take the expectation to complete the proof.

\subsection{Proof of Proposition~\ref{prop:pointwise-stability-main}}\label{subsec:app:prop3.2}

The proof is similar to the previous one thus we will only give the key step:  replace $Z=(X,Y)$ by $Z_i=(X_i,Y_i)$ in Equation~\eqref{eq:prepare-last-step}  to obtain:
\begin{align*}
\left|\loss\left(\alg\left(\DD\right),Z_i\right)-\loss\left(\alg\left(\DD^{\backslash i}\right),Z_i\right)\right|&\leq\sup_{\abs{y},\abs{y'}\leq \hat \rho^i_\lambda }\ffrac{ \alpha\left|\phi'(h_S(X_i)Y_i+y)\phi'(h_S(X_i)Y_i+y')\right|}{\ n}.
\end{align*} 

To conclude the proof, take the expectation of both sides of the last inequality and use the Assumption~\ref{Ass:4}.

\subsection{proof of Proposition~\ref{prop:loss-property}}\label{subsec:app:prop3.3}

First, let $i\in[n]$ and $\abs{y},\abs{y'}\leq \hat \rho^i_\lambda$. Furthermore let $Z=(X,Y)$ and $Z=(X',Y')$ be two observations independent of $\DD^{\backslash i}$.  We start by showing that the MSE and squared hinge verify Assumptions~\ref{Ass:3},~\ref{Ass:4} and explicit their corresponding function $\Psi_1 , \Psi_2$ . To do so, remind that: 
\begin{align}\label{ineq:square-radius-control}
                \left(\hat\rho^i_\lambda\right)^2&=\max(\hat r^i_\lambda,\hat r_\lambda)^2\nonumber\\
                 &\leq (\hat r^i_\lambda+ \hat r_\lambda)^2\nonumber\\
                 &\leq 2 \left(\hat r^i_\lambda\right)^2+2 \left( \hat r_\lambda\right)^2\nonumber\\
                 & = 2\alpha \left(\ER[h_S]+\ER^{\backslash i}[h_S]\right).
\end{align}

\subsubsection{MSE}
Recall the MSE loss $\phi(x)=(1-x)^2.$ For all $x\in \rset$, one has:
\begin{align}\label{ineq:MSE-main-ineq}
	\abs{\phi^\prime(x+y)}&= 2\abs{1-x-y}\nonumber\\
	&\leq 2\abs{1-x} + 2\abs{y} \nonumber\\
	&\leq 2\sqrt{\phi(x)}+2\hat \rho^i_\lambda.
\end{align}
Thus,
\begin{align*}
	\sup_{\abs{y'},\abs{y}\leq \hat \rho^i_\lambda} \left|\phi'(h_S(X')Y'+y')\phi'(h_S(X)Y+y)\right|&\leq 4\left(\sqrt{\phi\left(h_S(X')Y'\right)}+\hat \rho^i_\lambda\right)\left(\sqrt{\phi\left(h_S(X)Y\right)}+\hat \rho^i_\lambda\right).
\end{align*}

Taking the expectation of the latter inequality and using Lemma \ref{lemma:expectation-mult-control} with $X=\sqrt{\phi\left(h_S(X')Y'\right)}$ , $Y=\sqrt{\phi\left(h_S(X)Y\right)}$ and $Z=\hat \rho^i_\lambda$ yields:
\begin{align*}
\EE\left[\sup_{\abs{y'},\abs{y}\leq \hat \rho^i_\lambda} \left|\phi'(h_S(X')Y'+y')\phi'(h_S(X)Y+y)\right|\right]&\leq 8\left(\EE\left[\sqrt{\phi\left(h_S(X)Y\right)}\right]^2+\EE\left[\left(\hat \rho^i_\lambda\right)^2\right]\right)\\
(\text{by Jensen's Inequality})&\leq 8\left(\EE\left[\phi\left(h_S(X)Y\right)\right]+\EE\left[\left(\hat \rho^i_\lambda\right)^2\right]\right)\\
\left(\phi\left(h_S(X)Y\right)=\loss(h_S,Z)\right)&\leq 8\left(\risk[h_S]+\EE\left[\left(\hat \rho^i_\lambda\right)^2\right]\right)\\
(\text{Inequality~\eqref{ineq:square-radius-control}})&\leq 8\left(\risk[h_S]+4\alpha\risk[h_S]\right).
\end{align*}

This means that the MSE verifies Assumption~\ref{Ass:3} with $\Psi_1(x)=8x(1+4\alpha)$. Now using Inequality~\eqref{ineq:MSE-main-ineq} again yields:
\begin{align*}
	\sup_{\abs{y'},\abs{y}\leq \hat \rho^i_\lambda} \left|\phi'(h_S(X)Y+y')\phi'(h_S(X)Y+y)\right|&\leq 4\left(\sqrt{\phi\left(h_S(X)Y\right)}+\hat \rho^i_\lambda\right)^2.
\end{align*}
By taking the expectation and mimicking the previous step one can show that the MSE verifies Assumption~\ref{Ass:4} with $\Psi_2(x)=8x(1+4\alpha)$.
\subsubsection{Squared hinge} 

First recall the loss function $\phi(x)=\max\left(0,1-x\right)^2.$ By simple calculation we obtain:

$$\abs{\phi^\prime(x+y)}= 2\max\left(0,1-x-y\right).$$

On the other hand, one has:
\begin{equation*}
	\begin{cases}
		0  \leq \max(0,1-x) +\abs{y},\\
		1-x-y \leq \max(0,1-x)+\abs{y}.
	\end{cases}
\end{equation*}
Thus, it holds:
\begin{align*}
	\abs{\phi^\prime(x+y)}&\leq  2\max(0,1-x)+2\abs{y}\\
	&\leq 2\sqrt{\phi(x)}+2\hat \rho^i_\lambda.
\end{align*}
The result follows using the same steps as in the MSE case.

\subsubsection{Exponential} 

Recalling the loss function $\phi(x)=e^{-x}$, first notice that the exponential loss verifies:

\begin{equation}\label{ineq:main-exponential-ineq}
\abs{\phi^\prime(x+y)} = e^{-x}e^{-y}=\phi(x)e^{-y}\leq\phi(x)e^{\hat \rho^i_\lambda}\leq\phi(x)e^{\hat \tau^i_\lambda},
\end{equation}
where $\hat \tau^i_\lambda$ is given by Equation~\eqref{def:ball-prime}. Thus, we get:
\begin{align*}
		\EE\left[\sup_{\abs{y'},\abs{y}\leq \hat \rho^i_\lambda} \left|\phi'(h_S(X')Y'+y')\phi'(h_S(X)Y+y)\right|\right]&\leq \EE\left[\phi(h_S(X')Y')\phi(h_S(X)Y)e^{2\hat \tau^i_\lambda}\right]\\
( Z\indep Z' \indep \hat \tau^i_\lambda ) &\leq \risk[h_S]^2 \EE\left[e^{2\hat \tau^i_\lambda}\right].\\
(\text{By Lemma \ref{lemma:exponential-radius-control} with $p=2$})&\leq \risk[h_S]^2 e^{2+\frac{2\alpha M_S}{n}+\frac{4\alpha^2M_S^2}{n-1}}e^{2\alpha\risk[g]}
\end{align*}

Thus the exponential loss verifies Assumption \ref{Ass:3} with $\Psi_1(x)=C_Sx^2 e^{2\alpha x}$ and $C_S=e^{2+\frac{2\alpha M_S}{n}+\frac{4\alpha^2M_S^2}{n-1}}$. Besides, using~\eqref{ineq:main-exponential-ineq} again yields:
\begin{align*}
		\EE\left[\sup_{\abs{y'},\abs{y}\leq \hat \rho^i_\lambda} \left|\phi'(h_S(X)Y+y')\phi'(h_S(X)Y+y)\right|\right]&\leq \EE\left[\phi(h_S(X)Y)^2e^{2\hat \tau^i_\lambda}\right]\\
  (M_S=\sup_{Z\in\mathcal{Z}_T}\loss(h_S,Z))&\leq M_S\EE\left[\phi(h_S(X)Y)e^{2\hat \tau^i_\lambda}\right] \\ 
( Z \indep \hat \tau^i_\lambda ) &\leq M_S\risk[h_S] \EE\left[e^{2\hat \tau^i_\lambda}\right]\\
(\text{By Lemma \ref{lemma:exponential-radius-control} with $p=2$})&\leq M_S\risk[h_S] e^{2+\frac{2\alpha M_S}{n}+\frac{4\alpha^2M_S^2}{n-1}}e^{2\alpha\risk[g]}.
\end{align*}

Therefore the exponential loss verifies Assumption \ref{Ass:4} with $\Psi_2(x)=C_S M_S x e^{2\alpha x}.$
\subsubsection{Logistic} 

Recall the loss function $\phi(x)=\log(1+e^{-x})$ and its derivative:
$$\abs{\phi^\prime(x)}=\ffrac{e^{-x}}{e^{-x}+1}.$$

Thus, we have:
\begin{align*}
	\abs{\phi^\prime(x+y)}&=\ffrac{e^{-x-y}}{e^{-x-y}+1}\\
	&\leq e^{-y} e^{-x}\\
	&= e^{-y}\left(e^{\phi(x)}-1\right)\\
&\leq  e^{\hat \rho_\lambda^i}\left(e^{\phi(x)}-1\right)\\
  &\leq e^{\hat \tau_\lambda^i}\left(e^{\phi(x)}-1\right),
\end{align*}

where the two last inequalities result from the facts that $y\leq \hat \rho_\lambda^i$ and $  \hat \rho_\lambda^i\leq \hat \tau_\lambda^i$ respectively. Using the facts that $\norm{\phi'}_\infty\leq 1$ and $e^{\hat \tau_\lambda^i}\leq 1 $, one obtains:
\begin{align}\label{ineq:log-loss-key}
	\left|\phi'(h_S(X)Y+y)\right|&\leq \min\left(e^{\hat \tau_\lambda^i}\left(e^{\phi\left(h_S(X)Y\right)}-1\right),1\right)\nonumber\\
 &=\min\left(e^{\hat \tau_\lambda^i}\left(e^{\loss(h_S,Z)}-1\right),1\right)\nonumber\\
 &\leq \min\left(e^{\hat \tau_\lambda^i}\left(e^{\loss(h_S,Z)}-1\right),e^{\hat \tau_\lambda^i}\right)\nonumber\\
 &\leq e^{\hat \tau_\lambda^i}\min\left(\left(e^{\loss(h_S,Z)}-1\right),1\right).
\end{align}

The latter inequality yields:
\begin{align*}
\sup_{\abs{y'},\abs{y}\leq \hat \rho^i_\lambda} \left|\phi'(h_S(X')Y'+y')\phi'(h_S(X)Y+y)\right|\leq e^{2\hat \tau_\lambda^i}\min\left(e^{\loss\left(h_S,Z\right)}-1,1\right)\min\left(e^{\loss\left(h_S,Z'\right)}-1,1\right).
\end{align*}
Thus, since $Z, Z'$ are independent of $\DD_T^{\backslash i}$, they are also independent of $\hat \tau^i_\lambda$. It follows: 
 \begin{align}\label{ineq:log-loss-prepare}
 \EE\left[\sup_{\abs{y'},\abs{y}\leq \hat \rho^i_\lambda} \left|\phi'(h_S(X')Y'+y')\phi'(h_S(X)Y+y)\right| \right] &\leq \EE\left[ e^{2\hat \tau_\lambda^i}\right] \EE\left[\min\left(e^{\loss(h_S,Z)}-1,1\right)\right]^2\nonumber\\
(\text{By Lemma \ref{lemma:exponential-radius-control}})& \leq  C_Se^{2\alpha\risk[h_S]}\EE\left[\min\left(e^{\loss(h_S,Z)}-1,1\right)\right]^2.
\end{align}

Now using the fact that:

$$ e^{\loss(h_S,Z)}-1 \leq 1 \implies \loss(h_S,Z)\leq 1 \implies \loss(h_S,Z)\leq \sqrt{\loss(h_S,Z)}, $$

we have:
\begin{equation*}
\EE\left[\min\left(e^{\loss(h_S,Z)}-1,1\right)\right]\leq \EE\left[\min\left(e^{\sqrt{\loss(h_S,Z)}}-1,1\right)\right].
\end{equation*}
In addition, notice that:
\begin{equation*}
	(e^{\sqrt{x}}-1)\wedge 1 =
	\begin{cases}
	 e^{\sqrt{x}}-1 \quad &\text{if}\; x \leq \ln(2)^2,\\
	 1   \quad &\text{otherwise},
	\end{cases}
\end{equation*}
 which is concave. Therefore, it holds:
\begin{equation*}
\EE\left[\min\left(e^{\loss(h_S,Z)}-1,1\right)\right]\leq \min\left(e^{\sqrt{\risk[h_S]}}-1,1\right).
\end{equation*}
To show that the logistic loss verifies Assumption~\ref{Ass:3} with $\Psi_1(x)=C_Se^{2\alpha\risk[h_S]}(e^{\sqrt{x}}-1)^2$, it suffices to plug the latter inequality in \eqref{ineq:log-loss-prepare}. Now, using \eqref{ineq:log-loss-key} again yields:
\begin{align*}
 \EE\left[\sup_{\abs{y'},\abs{y}\leq \hat \rho^i_\lambda} \left|\phi'(h_S(X)Y+y')\phi'(h_S(X)Y+y)\right|\right] &\leq \EE\left[e^{2\hat \tau_\lambda^i}\min\left(e^{\loss(h_S,Z)}-1,1\right)^2\right]\\
 &\leq \EE\left[e^{2\hat \tau_\lambda^i}\right] \EE\left[\min\left(e^{\loss(h_S,Z)}-1,1\right)\right].
\end{align*}
Finally, using the same steps as before, we show that the logistic loss verifies Assumption~\ref{Ass:4} with $\Psi_1(x)=C_Se^{2\alpha\risk[h_S]}(e^{\sqrt{x}}-1)$.

\subsubsection{Softplus} 

The proof is similar to that of the logistic loss and is left for the reader.

\subsection{Proof of Theorem \ref{theo:gen-gap-UB}}\label{subsec:app:gen-gap-UB}

First, notice that:

\begin{align*}
\mathcal{E}_{\textrm{gen}}=\abs{\EE\left[\ER\left[\alg(\DD_T)\right]-\risk\left[\alg (\DD_T)\right]\right]}&=\abs{\EE\left[\frac{1}{n}\sum_{i=1}^{n}\loss\left(\alg\left(\DD_T\right),Z_i\right)-\loss\left(\alg\left(\DD_T^{\backslash i}\right),Z\right)\right]}\\
&=\abs{\EE\left[\loss\left(\alg\left(\DD_T\right),Z_1\right)-\loss\left(\alg\left(\DD_T^{\backslash i}\right),Z\right)\right]}.
\end{align*}

Using triangle inequality and the fact that $Z$ and $Z_1$ have the same distributions, we obtains:

\begin{align*}
    \mathcal{E}_{\textrm{gen}} &\leq \abs{\EE\left[\loss\left(\alg\left(\DD_T\right),Z_1\right)- \loss\left(\alg\left(\DD^{\backslash i}_T\right),Z_1\right)\right]}+\abs{\EE\left[\loss\left(\alg\left(\DD_T^{\backslash i}\right),Z_1\right)-\loss\left(\alg\left(\DD_T^{\backslash i}\right),Z\right)\right]}\\
&= \abs{\EE\left[\loss\left(\alg\left(\DD_T\right),Z_1\right)- \loss\left(\alg\left(\DD^{\backslash i}_T\right),Z_1\right)\right]}+\abs{\EE\left[\loss\left(\alg\left(\DD_T^{\backslash i}\right),Z\right)-\loss\left(\alg\left(\DD_T^{\backslash i}\right),Z\right)\right]}.
\end{align*}

The desired result follows from Propositions~\ref{prop:stability-main} and~\ref{prop:pointwise-stability-main}. 

\subsection{Proof of Theorem \ref{theo:custom-stability-surrogate}}\label{subsec:excess-risk}

First introduce $$h_\lambda=\argmin_{h\in \mathcal{H}}\risk\left[h_S+h\right]+\lambda\norm{h}_k^2, $$ and write
\begin{align*}
	 \risk\left[\alg\right] -\risk[h^*+h_S] &=\risk\left[\alg\right]-\ER\left[\alg\right]
	 +\ER\left[\alg\right] +\lambda\norm{f}_k^2  -\risk[h_\lambda+h_S] +\risk[h_\lambda+g']  -\risk[h^*+h_S].
\end{align*}
Now by rearranging and reminding that: 
$$\ER\left[\alg\right]+\lambda\norm{f}_k^2 \leq \ER\left[h_\lambda+h_S\right]+\lambda\norm{h_\lambda}_k^2,$$

we obtain:
\begin{align*}
	\risk\left[\alg\right] -\risk[h^*+h_S] &\leq\risk\left[\alg\right]-\ER\left[\alg\right] +\ER\left[h_\lambda+h_S\right] -\risk[h_\lambda+h_S]   \\&  \hspace{1cm}+\risk[h_\lambda+h_S]+\lambda\norm{h_\lambda}_k^2-\risk[h^*+h_S].
\end{align*}
For the first term notice that:
\begin{align*}
	\EE\left[\risk\left[\alg\right]-\ER\left[\alg\right]\right]
	&\leq \mathcal{E}_{\mathrm{gen}}\leq \beta(n)+\gamma(n).
\end{align*}
Regarding the second term, since $h_\lambda$ is independent of $\DD_T$ we have: 
$$ \EE\left[\ER\left[h_\lambda+h_S\right]-\risk[h_\lambda+h_S]\right]=0.$$
Finally, notice that by definition of $g_\lambda$ that:
$$\risk[h_\lambda+h_S]+\lambda\norm{h_\lambda}_k^2-\risk[h^*+h_S]\leq \lambda\norm{h^*}_k^2.$$
Combining the latter four inequalities yields:
\begin{equation}\label{ineq:excess-app}
    \mathcal{E}_{\mathrm{ex}}=\risk\left[\alg\right] -\risk[h^*+h_S]\leq \beta(n)+\gamma(n)+\lambda\norm{h^*}_k,
\end{equation}

which concludes the first part. For the second part we shall use Table~\ref{table:property-hypothesis} and the fact that 
\begin{equation}\label{ineq:key-ineq-excess}
\gamma(n)+\beta(n)\leq \alpha\ffrac{\left(\Psi_1\left(\risk\left[h_S\right]\right)+  \Psi_2\left(\risk\left[h_S\right]\right)\right)\wedge \left(2\norm{\phi'}_\infty^2\right)}{ n}.
\end{equation}

\subsubsection{MSE and Squared hinge }

For these two losses,  $\Psi_1(x)=\Psi_2(x)=8x(4\alpha+1)$, so that by inequality \eqref{ineq:key-ineq-excess} we get:
\begin{align*}
\gamma(n)+\beta(n)&\leq  \alpha\ffrac{16\risk\left[h_S\right](4\alpha+1)}{ n}\\
&=\ffrac{16\kappa \risk\left[h_S\right](4\frac{\kappa}{\lambda}+1)}{ \lambda n}.
\end{align*}
Thus for small $\lambda$ one has:  $$\gamma(n)+\beta(n)=\mathcal{O}\left(\frac{\risk[h_S]}{\lambda^2 n}\right).$$

To conclude, set $\lambda=\sqrt{\ffrac{\risk[h_S]}{\sqrt{n}}}$ and use Inequality~\eqref{ineq:excess-app} to obtain:
$$\mathcal{E}_{\textrm{ex}}=\mathcal{O}\left( \sqrt{\ffrac{\risk[h_S]}{\sqrt{n}}} \right).$$

\subsubsection{Exponential}

Using Table~\ref{table:property-hypothesis}, remind that the functions $\Psi_1(x)$ and 
$\Psi_2$ are given by: $\Psi_1(x)=C_Sx^2e^{2\alpha x}$ , $\Psi_2(x)=C_SM_Sxe^{2\alpha x}$ with $M_{S}=\sup_{z\in \mathcal{Z}_T}\loss(h_S,z)$ and

$$C_S=\exp\left\{2 +\frac{2\alpha  M_S}{n}+\frac{4\alpha^2  M_S^2}{n-1}\right\}=\exp\left\{2 +\frac{2\kappa  M_S}{\lambda n}+\frac{4\kappa^2  M_S^2}{\lambda^2(n-1)}\right\}.$$

Assume that $n\geq \max\left(\frac{M_S^2\ln(n)^2} {\risk[h_S]},2\right)$ and $\lambda=4\frac{\sqrt{\risk[h_S]}\wedge 1}{\ln(n)}=4\frac{\sqrt{\risk[h_S]}}{\ln(n)}$. The case where $\risk[h_S]\geq 1$ is similar and thus omitted. Now, write
$$n\geq \frac{M_S^2\ln(n)^2} {\risk[h_S]}= \frac{M_S^2} {\lambda^2}\implies \frac{  M_S^2}{\lambda^2(n-1)}\leq \frac{n}{n-1}\leq 2. $$

The latter condition also implies that $\frac{  M_S}{\lambda n}\leq \frac{\lambda}{M_S}\leq \frac{\sqrt{\risk[h_S]}}{M_S\ln(n)}\leq \frac{1}{2 M_S}$. By these two facts, we deduce that  $C_S$ can be bounded independently of $n$. Thus, using \eqref{ineq:key-ineq-excess} yields:
\begin{align}\label{ineq:last-step-excess}
\gamma(n)+\beta(n)&\leq  \alpha\ffrac{C_S\left(\risk[h_S]^2+M_S\risk[h_S]\right)e^{2\alpha\risk[h_S]}}{ n}\\
&=\ln(n)\ffrac{C_S\left(\risk[h_S]^{3/2}+M_S\risk[h_S]^{1/2}\right)(\sqrt{n})^{\kappa\sqrt{\risk[h_S]}}}{ n}\nonumber\\
&=\mathcal{O}\left(\frac{\sqrt{\risk[h_S]}}{\sqrt{n}}\right),\nonumber
\end{align}

where the two last inequalities follow from the facts that $\alpha=\frac{\kappa}{\lambda}=\frac{\kappa\ln(n)}{\sqrt{\risk[h_S]}}$ and $\kappa \leq 1$. It remains to use the \eqref{ineq:excess-app} to conclude the first part. For the second part,  set $\lambda=\frac{\ln(n)^2}{\sqrt{n}}$ and notice that,
if $n\leq \frac{M_S^2\ln(n)^2} {\risk[h_S]}$ then $\risk[h_S]\leq \frac{M_S^2\ln(n)^2} {n}\leq M_S^2 \lambda$ and $\alpha\risk[h_s] \leq M_S^2 \kappa\leq 1$.  Furthermore, the constant $C_S$ can be bounded independently of $n$ with such a choice of $\lambda$. Inequality~\eqref{ineq:last-step-excess} becomes: 

$$\gamma(n)+\beta(n)=\mathcal{O}\left(\ffrac{\risk[h_S]}{\sqrt{n}\ln(n)^2}\right).$$

It remains to use Inequality \eqref{ineq:excess-app} to complete the proof.

\subsubsection{Logistic }
For this loss, we have $\norm{\phi'}_\infty =1 $ and Inequality \eqref{ineq:key-ineq-excess} becomes:

$$\beta(n)+\gamma(n)\leq \frac{2\alpha}{n}= \frac{2\kappa}{\lambda n}\leq \frac{2}{\lambda n}.$$
Thus, setting $\lambda=\frac{1}{\sqrt{n}}$  and using Inequality \eqref{ineq:excess-app} yields:

$$\mathcal{E}_{\textrm{ex}}=\mathcal{O}\left( \sqrt{\ffrac{1}{n}} \right).$$
Furthermore, if $n\geq 9$ and $\risk[h_S]\leq \frac{1}{\sqrt{n}}\leq \frac{1}{e}$, then with the choice $\lambda= \frac{8}{\sqrt{-n\ln(\risk[h_S])}}$ one has:
$$\mathcal{E}_{\textrm{ex}}=\mathcal{O}\left(\ffrac{1}{ \sqrt{-n\ln\left(\risk[h_S]\right) }}\right).$$
Indeed, in the setting above, it leads to that:
\begin{align*}
e^{\alpha\risk[h_S]}=e^{\frac{\kappa}{\lambda} \risk[h_S]}
&= e^{\frac{\kappa\sqrt{-\ln(\risk[h_s])}}{8}}\\
(\kappa \leq 1)&\leq e^{\frac{\sqrt{-\ln(\risk[h_s])}}{8}}\\
\left(-\ln(\risk[h_S]) \geq 1\right)  &\leq e^{\frac{-\ln(\risk[h_s])}{8}} =\risk[h_S]^{-1/8},
\end{align*}
and
$$e^{\frac{2\kappa  M_S}{\lambda n}}\leq e^{\frac{2}{\lambda n}}=e^{-\frac{\sqrt{\ln(\risk[h_S]})}{4\sqrt{n}}}\leq e^{\sqrt{\frac{\ln(n)}{4n}}}\leq e^{1/
4}.$$
Besides, since $\frac{n}{n-1}\leq 2$, 
$$e^{\frac{4\kappa  M_S}{\lambda^2 n}}\leq e^{\frac{4}{\lambda^2 (n-1)}}=e^{-\frac{\ln(\risk[h_S])(n)}{16(n-1)}}\leq \risk[h_S]^{-1/8}.$$
Moreover, using Inequality \eqref{ineq:key-ineq-excess} and Table \ref{table:property-hypothesis} gives:
\begin{align*}
\gamma(n)+\beta(n)&\leq  \alpha\ffrac{C_Se^{\sqrt{\risk[h_S]}}e^{2\alpha\risk[h_S]}\left(e^{\sqrt{\risk[h_S]}}-1\right)}{ n}\\
&=\kappa\ffrac{\exp\left\{2 +\frac{2\alpha  M_S}{n}+\frac{4\alpha^2  M_S^2}{n-1}+\sqrt{\risk[h_S]}+2\alpha\risk[h_S]\right\}\left(e^{\sqrt{\risk[h_S]}}-1\right)}{\lambda n}\\
&=\mathcal{O}\left(\ffrac{\sqrt{-\ln\left(\risk[h_S]\right)}\left(e^{\sqrt{\risk[h_S]}}-1\right)}{\risk[h_S]^{1/4} \sqrt{n}}\right).
\end{align*}

Now, since the function $e^{\sqrt{x}}-1\leq 2\sqrt{x}$ for all $x\leq \ln(2)^2$ and $\risk[h_S]\leq \frac{1}{n}\leq \frac{1}{3} \leq \ln(2)^2$ the latter inequality becomes:
\begin{align*}
\gamma(n)+\beta(n)=\mathcal{O}\left(\ffrac{\sqrt{-\ln\left(\risk[h_S]\right)}\risk[h_S]^{1/4}}{ \sqrt{n}}\right).
\end{align*}

To conclude the proof notice that, for all $x\leq 1$, we have $\ln(x^{-1/4})\leq x^{-1/4}$ and thus $x^{1/4}(-\ln(x))\leq 4$. This leads to :

\begin{align*}
x^{1/4}\sqrt{-\ln(x)}\leq \frac{4}{\sqrt{-\ln(x)}}.
\end{align*}
Therefore,
\begin{align*}
\gamma(n)+\beta(n)=\mathcal{O}\left(\ffrac{1}{ \sqrt{-n\ln\left(\risk[h_S]\right)}}\right).
\end{align*}

It remains to use Inequality \eqref{ineq:excess-app} to complete the proof.
\subsubsection{Softplus }

For the softplus, the choice $\lambda=1/\sqrt{n}$ yields:
$$\mathcal{E}_{\textrm{ex}}=\mathcal{O}\left( \sqrt{\ffrac{1}{n}} \right).$$
Furthermore, if  $n\geq 9$ and $\risk[h_S]\leq \frac{1}{\sqrt{n}}$ and $\frac{1}{s}\leq -\ln(\risk[h_S]$), then with the choice $\lambda= \frac{8}{\sqrt{-n\ln(\risk[h_S])}}$ , one has:
$$\mathcal{E}_{\textrm{ex}}=\mathcal{O}\left(\ffrac{1}{ \sqrt{-s n\ln\left(\risk[h_S]\right)}}\right).$$

The proof is identical to the previous one and thus omitted.

\end{document}